\pdfoutput=1
\documentclass{article}

\PassOptionsToPackage{pdftex}{graphicsx}

\usepackage[preprint]{neurips_2022}

\usepackage{natbib}
\setcitestyle{authoryear,open={(},close={)}}

\usepackage[utf8]{inputenc} 
\usepackage[T1]{fontenc}    
\usepackage[hidelinks]{hyperref}       
\usepackage{url}            
\usepackage{booktabs}       
\usepackage{amsfonts}       
\usepackage{nicefrac}       
\usepackage{microtype}      
\usepackage{xcolor}         
\usepackage{xfrac}
\usepackage{bm}

\usepackage{amsmath}
\usepackage{amsfonts}
\usepackage{color}
\usepackage[makeroom]{cancel}
\usepackage{cancel}
\usepackage{tabu}
\usepackage[normalem]{ulem} 
\usepackage{framed}
\usepackage{subcaption}
\usepackage{wrapfig}
\usepackage{amssymb}
\usepackage{pifont}
\usepackage{multirow}
\usepackage{booktabs}
\usepackage[flushleft]{threeparttable}
\usepackage{float}
\usepackage[capitalize]{cleveref}
\usepackage{setspace}
\usepackage{caption}
\usepackage{multicol}
\usepackage{algpseudocode}
\usepackage{algorithm}
\usepackage{amsthm}
\newtheorem{theorem}{Theorem}

\def\N{\mathcal{N}}

\def\x{{\mathbf x}}

\def\q{{\mathbf q}}

\def\0{{\mathbf 0}}

\def\psib{{\boldsymbol \psi}}
\def\phib{{\boldsymbol \phi}}
\def\thetab{{\boldsymbol \theta}}

\newcommand{\Dcal}{\mathcal{D}}

\newcommand{\xc}{\bm{x}}
\newcommand{\yc}{\bm{y}}
\newcommand{\zc}{\bm{z}}

\newcommand{\Hb}{\mathbf{H}}

\crefname{section}{Sec.}{Secs.}
\Crefname{section}{Section}{Sections}
\Crefname{table}{Table}{Tables}
\crefname{table}{Tab.}{Tabs.}

\newcommand{\update}[1]{{\color{black}#1}}

\title{Laplacian Autoencoders for \\Learning Stochastic Representations} 

\author{%
  Marco Miani$^1$, \ \ Frederik Warburg$^1$, \\ \textbf{Pablo Moreno-Muñoz, \ \ Nicke Skafte Detlefsen, \ \ Søren Hauberg} \\
  \texttt{\{mmia, frwa, pabmo, nsde, sohau\}@dtu.dk} \\
  Technical University of Denmark\\
  \\
  \texttt{\url{https://github.com/FrederikWarburg/LaplaceAE}}
}
\begin{document}


\maketitle

\begin{abstract}
    Established methods for unsupervised representation learning such as variational autoencoders produce none or poorly calibrated uncertainty estimates making it difficult to evaluate if learned representations are stable and reliable. In this work, we present a Bayesian autoencoder for unsupervised representation learning, which is trained using a novel variational lower bound of the autoencoder evidence. This is maximized using Monte Carlo EM with a variational distribution that takes the shape of a Laplace approximation. We develop a new Hessian approximation that scales linearly with data size allowing us to model high-dimensional data. Empirically, we show that our Laplacian autoencoder estimates well-calibrated uncertainties in both latent and output space. We demonstrate that this results in improved performance across a multitude of downstream tasks.
\end{abstract}

\let\thefootnote\relax\footnotetext{$^1$ Denotes equal contribution\update{; author order determined by a simulated coin toss}.}
\everypar{\looseness=-1}
\section{Introduction}
Unsupervised representation learning is a brittle matter. Consider the \emph{classic autoencoder (\textsc{ae})}
\citep{rumelhart1986learning}, which compresses data $\xc$ to a low-dimensional representation $\zc$ from which data is approximately reconstructed. The nonlinearity of the model implies that sometimes small changes to data $\xc$ \update{give} a large change in the latent representation $\zc$ (and sometimes not). Likewise, for some data, reconstructions are of low quality, while for others it is near perfect. Unfortunately, the model does not have a built-in quantification of its uncertainty, and we cannot easily answer when the representation is reliable and accurately reflects data.

The celebrated \emph{variational autoencoder} (\textsc{vae}) \citep{kingma2014vae, rezende2014vae} address this concern directly through an explicit likelihood model $p(\xc | \zc)$ and a variational approximation of the representation posterior $p(\zc | \xc)$. Both these distributions have parameters predicted by neural networks that act similarly to the encoder--decoder pair of the classic autoencoder.

But is the \update{\textsc{vae}'s} quantification of reliability reliable? To investigate, we fit a \textsc{vae} with a two-dimensional latent representation to the \textsc{mnist} dataset \citep{lecun1998mnist}, and illustrate the predicted uncertainty of $p(\xc | \zc)$ in \cref{fig:teaser}a. The model learns to assign high uncertainty to low-level image features such as edges but predicts its smallest values far away from the data distribution. Not only is such behavior counter-intuitive, \update{but} it is also suboptimal in terms of data likelihood (Sec.~\ref{sec:background}). Retrospectively, this should not be surprising as the uncertainty levels away from the data are governed by the extrapolatory behavior of the neural network determining $p(\xc | \zc)$. This suggests that perhaps uncertainty should be a derived quantity rather than a predicted one.

From a Bayesian perspective\update{,} the natural solution is to form an (approximate) posterior over the weights of the neural networks. To investigate, we \update{adapt} a state-of-the-art implementation of a \emph{post-hoc} Laplace approximation \citep{daxberger2021laplace} of the weight posterior \update{to the autoencoder domain}. This amounts to training a regular autoencoder, and thereafter approximating the weight uncertainty with the Hessian of the loss (Sec.~\ref{sec:background}). \cref{fig:teaser}b shows that uncertainty now grows, as intuitively expected, with the distance to the data distribution, but there seems to be little semantic structure in the uncertainty in output space. This suggests that while the post-hoc procedure is computationally attractive it is too simplistic.~\looseness=-1

\textbf{In this paper} we introduce a new framework for Bayesian autoencoders in unsupervised representation learning. Our method takes inspiration from the Laplace approximation to build a variational distribution on the neural network weights. We first 
\update{propose a post-hoc \textsc{la} for autoencoders; showcasing good out-of-distribution detection capabilities, but lack of properly calibrated uncertainties in-distribution.}
\update{To address this,} we develop a fast and memory-efficient Hessian approximation, which allows us to maximize a variational lower bound using Monte Carlo EM, such that model uncertainty is a key part of model training rather than estimated post-hoc. \cref{fig:teaser}c gives an example of the corresponding uncertainty, which exhibits a natural and semantically meaningful behavior.\looseness=-1
\begin{figure}[t]
    \centering
    \includegraphics[width=\textwidth]{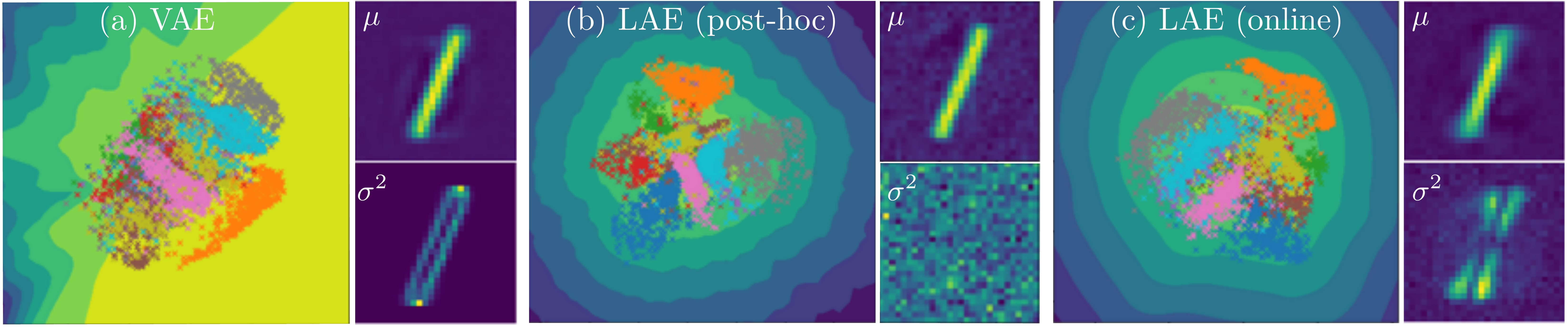}
    \caption{\update{2D latent representation of \textsc{mnist} overlaid a heatmap that describes the decoder uncertainty (yellow/blue indicates a low/high variance of the reconstructions). To the right of the latent spaces, we show the mean and variance of a reconstructed image (yellow indicates high values).} (a) The \textsc{vae} learns to estimate high variance for low-level image features such as edges but fails at extrapolating uncertainties away from training data. (b) Applying post-hoc Laplace to the \textsc{ae} setup shows much better extrapolating capabilities, but fails in estimating calibrated uncertainties in output space. (c) Our online, sampling-based optimization of a Laplacian autoencoder (\textsc{lae}) gives well-behaved uncertainties in both latent and output space.}
    \label{fig:teaser}
\end{figure}

\subsection{Background}\label{sec:background}

\textbf{The \textsc{vae}} is a latent variable model that parametrize the data density $p(\xc) = \int p(\xc | \zc) p(\zc) \mathrm{d}\zc$ using a latent variable (representation) $\zc$. Here $p(\zc)$ is a, usually standard normal, prior over the representation, and $p(\xc | \zc)$ is a likelihood with parameters predicted by a neural network.

The nonlinearity of the likelihood parameters \update{renders} the marginalization of $\zc$ intractable, and a variational lower bound of $p(\xc)$ is considered instead. To arrive at this, one first introduces a variational approximation $q(\zc | \xc) \approx p(\zc | \xc)$ and write $p(\xc) = \mathbb{E}_{q(\zc|\xc)}\left[ p(\xc | \zc) \sfrac{p(\zc)}{q(\zc | \xc)} \right]$. A lower bound on $p(\xc)$ then follows by a direct application of Jensen's inequality,
\begin{align}
  \update{\log p(\xc)} \geq \mathcal{L}_{\text{\textsc{vae}}}(\xc) = \mathbb{E}_{q(\zc|\xc)}\left[ \log p(\xc|\zc) \right] - \text{KL}(q(\zc|\xc) \| p(\zc)).
  \label{eq:vae_elbo}
\end{align}
If we momentarily assume that $p(\xc | \zc) = \N(\xc | \mu(\zc), \sigma^2(\zc))$, we see that optimally $\sigma^2(\zc)$ should be as large as possible away from training data in order to increase $p(\xc)$ on the training data \update{(\cref{sec:vae_optim})}. Yet this is not the observed empirical behavior in Fig.~\ref{fig:teaser}a. Since the $\sigma^2$ network is left untrained away from training data, its predictions \update{depend} on extrapolation. In practice, $\sigma^2$ \update{takes} fairly small values near training data (assuming the mean $\mu$ provides a reasonable data fit), and $\sigma^2$ \update{extrapolates arbitrary} even if this is suboptimal in terms of data likelihood. Similar remarks hold for other likelihood models $p(\xc|\zc)$ and encoder distributions $q(\zc|\xc)$: \emph{relying on neural network extrapolation to predict uncertainty does not work}.

\begin{figure}[th!]
    \centering
    \includegraphics[width=\textwidth]{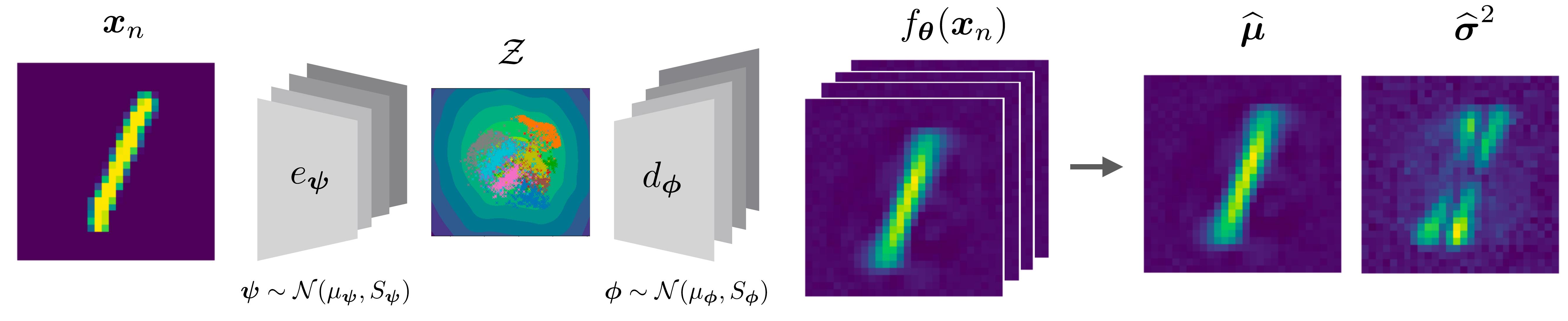}
    \vspace{-5mm}
    \caption{\textbf{Model overview.} We learn a distribution over parameters such that we can sample encoders $e_\psib$ and decoders $d_\phib$. This allow us to compute the empirical mean and variance in both the latent space $z$ and the output space $f_\thetab(\xc_n) = d_\phib(e_\psib(\xc_n))$.}
    \label{fig:overview}
\end{figure} 

\textbf{The Laplace approximation} \citep{laplace1774memoire,mackay1992laplace} is an integral part of our proposed solution. In the context of Bayesian neural networks, we seek the weight-posterior $p(\theta | \mathcal{D}) \propto \exp(-\mathcal{L}(\mathcal{D}; \theta))$, where $\theta$ are network weights, $\mathcal{D}$ is the training data, and $\mathcal{L}$ is the applied loss function interpreted as an unnormalized log-posterior. This is generally intractable and Laplace's approximation (\textsc{la}) amounts to a second-order Taylor expansion around a chosen weight vector $\theta^*$
\begin{align}\label{eq:taylor_laplace}
  \log p(\thetab | \mathcal{D})
    & = \mathcal{L}^*
    + (\thetab - \thetab^*)^{\top} \nabla \mathcal{L}^*
    + \frac{1}{2} (\thetab - \thetab^*)^{\top} \nabla^2 \mathcal{L}^* (\thetab - \thetab^*) + \mathcal{O}(\|\thetab - \thetab^*\|^3)
\end{align}
where we use the short-hand $\mathcal{L}^* = \mathcal{L}(\mathcal{D}; \theta^*)$. The approximation, thus, assumes that $p(\theta | \mathcal{D})$ is Gaussian. Note that when $\theta^*$ is a \textsc{map} estimate, the first order term vanishes and the second order term is negative semi-definite. We provide more details on the Laplace approximation in \cref{sec:laplace_approx_appendix}. In practice, computing the full Hessian is too taxing both in terms of computation and memory, and various approximations are applied (Sec.~\ref{sec:hessian}).

\section{Laplacian Autoencoders}\label{sec:method}
We consider unsupervised representation learning from i.i.d.\@ data $\Dcal\!=\!\{\xc_n\}^{N}_{n=1}$ consisting of observations $\xc_n \in \mathbb{R}^{D}$. We also define a continuous latent space such that representations $\zc_n\!\in\!\mathbb{R}^{K}$. Similar to \textsc{ae}s \citep{hinton2006reducing}, we consider two neural networks $e_{\psib}:\mathbb{R}^{D}$$\rightarrow$$\mathbb{R}^{K}$ and $d_{\phib}$$:\mathbb{R}^{K}$$\rightarrow$$\mathbb{R}^{D}$, widely known as the \emph{encoder} and \emph{decoder}. These have parameters $\thetab=\{\psib,\phib\}$. We refer to the composition of encoder and decoder as $f_\thetab = d_\psib \circ e_\phib$.

\textbf{Model overview.}~~~The autoencoder network structure implies that we model the data as being distributed on a $K$-dimensional manifold parametrized by $\thetab$. We then seek the distribution of the \emph{reconstruction} $\xc_{\text{rec}} = f_{\thetab}(\xc)$ given observation $\xc$, where the uncertainty comes from $\thetab$ being unknown,\looseness=-1
\begin{align}
    p(\xc_{\text{rec}}|\xc,f)
    & =
    \mathbb{E}_{\thetab \sim p(\thetab|\xc,f)} [p(\xc_{\text{rec}}|\thetab,\xc,f)].
    \label{eq:marginal_xrec}
\end{align}
Notice that we explicitly condition on $f$, which is the operator $\thetab \mapsto f_{\thetab}$, even if this is not stochastic; this \update{conditioning will become important later on to distinguish between the distribution deduced by $f$ and its linearization $f^{(t)}$.} 
Mimicking the standard autoencoder reconstruction loss, we set $p(\xc_{\text{rec}}|\thetab,\xc,f)\!=\!\mathcal{N}(\xc_{\text{rec}}|f_\thetab(x),\mathbb{I})$. Since $p(\thetab|\xc,f)$ is unknown, the reconstruction likelihood~\eqref{eq:marginal_xrec} is intractable, and approximations are in order. Similar to \citet{blundell2015bayesbybackprop}, we resort to a Gaussian approximation, but rather than learning the variance variationally, we opt for \textsc{la}. This will allow us to sample \textsc{nn}s and deduce uncertainties in both latent and output space as illustrated in~\cref{fig:overview}.

\textbf{Intractable joint distribution.}~~~Any meaningful approximate posterior over $\thetab$ should be similar to the marginal of the joint distribution $p(\thetab, \xc_{\text{rec}} | \xc, f)$.
This marginal is 
\begin{align}
    p(\thetab|\xc,f)
    =
    \mathbb{E}_{\xc_{\text{rec}}\sim p(\xc_{\text{rec}}|\xc,f)} [p(\thetab|\xc_{\text{rec}},\xc,f)]
    \label{eq:marginal_thetab}
\end{align}
which can be bounded on a log-scale using Jensen's inequality,
\begin{align}
    \log p(\thetab|\xc,f)
    \geq
    \mathcal{L}_{\thetab}
    =
    \mathbb{E}_{\xc_{\text{rec}}\sim p(\xc_{\text{rec}}|\xc,f)} [\log p(\thetab|\xc_{\text{rec}},\xc,f)].
    \label{eq:marginal_thetab_log}
\end{align}
Our first approximation is a \textsc{la} of $p(\thetab|\xc,f) \approx q^t(\thetab|\xc,f) = \mathcal{N}(\thetab | \thetab_t, \mathbf{H}_t^{-1})$, where we postpone the details on how to acquire $\thetab_t$ and $\mathbf{H}_t$. These will eventually be iteratively computed from the lower bound \eqref{eq:marginal_thetab_log}; hence the $t$ index. \cref{fig:linear_vs_nonlinear} (a) illustrates the situation thus far: $p(\thetab|\xc,f)$ is approximately Gaussian, but the non-linearity of $f$ gives $p(\xc_{\text{rec}}|\xc,f)$ a non-trivial density. 

\textbf{Linearization for gradient updates.}~~~Standard gradient-based learning can be viewed as a linearization of $f$ in $\thetab$, i.e.\@ for a loss $\mathcal{L} = l(f_{\thetab})$, the gradient is $\nabla_{\thetab} \mathcal{L} = J_{\thetab} f_{\thetab}\, \nabla_{f} l(f)$, where $J_{\thetab} f_{\thetab}$ is the Jacobian of $f$. In a similar spirit, we linearize $f$ in $\thetab$ in order to arrive at a tractable approximation of $p(\xc_{\text{rec}}|\xc,f)$. Specifically, we perform a Taylor expansion around $\thetab_t$
\begin{align}
    f_{\thetab}(\xc) 
      &= 
      \underbrace{
      f_{\thetab_t}(\xc) + J_{\thetab} f_{\thetab_t}(\xc) (\thetab - \thetab_t)
      }_{=:f^{(t)}_\thetab(\xc)}
      + \mathcal{O}\left(\| \thetab - \thetab_t \|^2\right)
  \label{eq:linearize_in_theta}
\end{align}
where $f^{(t)}$ denote the associated first-order approximation \update{of $f$}. Under this approximation, the joint distribution $p(\thetab, \xc_{\text{rec}} | \xc, f)$ becomes Gaussian, $p(\thetab, \xc_{\text{rec}} | \xc, f^{(t)}) \approx \mathcal{N}\left(\thetab, \xc_{\text{rec}} | \mu_{t}, \Sigma_{t}\right)$, with
\begin{align}
   \mu_{t}
      &= \begin{pmatrix}
           \thetab_t \\
           f^{(t)}_{\thetab_t} (\xc)
         \end{pmatrix},
   \text{ and }
   \Sigma_{t} = \begin{pmatrix}
     \mathbf{H}_t^{-1} & J_{\thetab} f_{\thetab_t}(\xc)^{\top} \\
     J_{\thetab} f_{\thetab_t}(\xc) & \left( J_{\thetab} f_{\thetab_t}(\xc)^{\top} \mathbf{H}_t J_{\thetab} f_{\thetab_t}(\xc) \right)^{-1} + \mathbb{I}
   \end{pmatrix}.
   \label{eq:covariance_update}
\end{align}
This approximation is illustrated in \cref{fig:linear_vs_nonlinear}(b). \update{We provide the proof in \cref{sec:proof_covariance}.}

\begin{figure}
    \centering
    \includegraphics[width=0.9\linewidth]{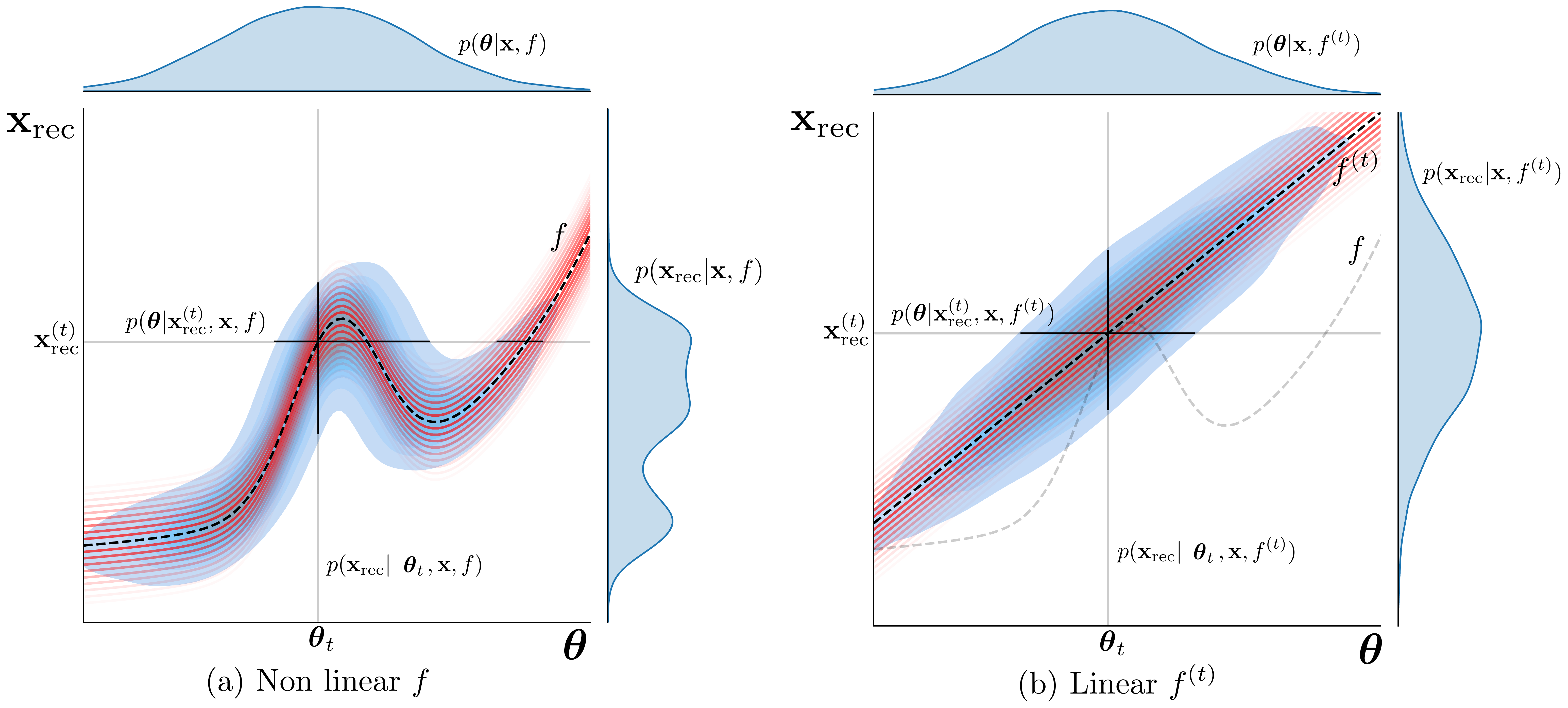}
     \caption{Illustrative example for fixed $\xc$. The likelihood for a fixed $\thetab_t$, shown by the columns, are assumed Gaussian $p(\xc_{\text{rec}}|\thetab_t, \xc,f) = N(f_{\thetab_t}(\xc), \mathbb{I})$. We model the marginalised density $p(\thetab | \xc, f)$ (first axis) over parameters $\thetab$ with Gaussians. With the additional assumption of linear $f$ then $p(\xc_{\text{rec}}|\xc, f)$ (second axis) is Gaussian. This makes the joint distribution tractable. }
     \label{fig:linear_vs_nonlinear}
\end{figure}

\textbf{Iterative learning.}~~~With these approximations we can readily develop an iterative learning scheme which updates $q(\thetab|\xc,f)$. The mean of this approximate posterior can be updated according to a standard variational gradient step, $\thetab_{t+1} = \thetab_t + \lambda \nabla_{\thetab} \mathcal{L}_{\xc_{\text{rec}}}$, where
\begin{align}
    \mathcal{L}_{\xc_{\text{rec}}} = \mathbb{E}_{\thetab \sim q^t(\thetab|\xc,f)} [\log p(\xc_{\text{rec}}|\thetab,\xc,f^{(t)})],
\end{align}
is a lower bound on Eq.~\ref{eq:marginal_xrec}, which we evaluate with a single Monte Carlo sample. \update{Following the \textsc{la},} 
the covariance of $q^{t+1}(\thetab|\xc,f)$, 
should be the inverse of the Hessian of $\log p(\thetab|\xc,f)$ at $\thetab_{t+1}$. Since this is intractable, we instead compute the Hessian of the lower bound \eqref{eq:marginal_thetab_log}
\begin{equation}
    \mathbf{H}_{t+1}
      \!=\! -\nabla^2_{\thetab} \mathcal{L}_{\thetab} \big|_{\thetab_{t+1}}
       \!=\! \mathbb{E}_{p(\xc_{\text{rec}}|\xc, f^{(t)})} 
         \!\!\left[ 
           -\nabla^2_{\thetab} \log p(\xc_{\text{rec}}|\thetab, \xc, f^{(t)}) -
           \!\nabla^2_{\thetab} \log q^t (\thetab | \xc, f)
         \right]\!\Big|_{{\thetab_{t+1}}}
\end{equation}
The last term can be approximated since $\nabla^2_{\thetab}\log q^t(\thetab|\xc,f) |_{\thetab=\thetab_{t+1}}=-\mathbf{H}_t+\mathcal{O}\left(\|\thetab_{t+1}-\thetab_t\|\right)$.

To efficiently cope with the $\mathcal{O}$-term, we introduce a parameter $\alpha$, such that the final approximation is
\begin{align}
    \mathbf{H}_{t+1}
      &\approx
      (1 - \alpha) \mathbf{H}_t + \mathbb{E}_{p(\xc_{\text{rec}}|\xc, f)} \left[ 
           -\nabla^2_{\thetab} \log p\left(\xc_{\text{rec}}\big|\thetab, \xc, f^{(t)}\right)
      \right]\Big|_{\thetab_{t+1}} \\
      &=
      (1 - \alpha) \mathbf{H}_t -
           {J_{\thetab} f^{(t)}_\thetab}^{\top} \nabla^2_{\xc_{\text{rec}}} \log p\left(\xc_{\text{rec}} \big| \thetab_{t+1}, \xc, f^{(t)}\right) {J_{\thetab} f^{(t)}_\thetab}.
\end{align}
\update{where $J_{\thetab} f^{(t)}_\thetab$ is independent of which $\theta$ we evaluate in and $\nabla^2_{\xc_{\text{rec}}} \log p\left(\xc_{\text{rec}} \big| \thetab_{t+1}, \xc, f^{(t)}\right)$ is trivial to compute for common losses, i.e.\@ for MSE it is the identity. The parameter $\alpha$ can be viewed as a geometric running average that is useful for smoothing out results computed on a minibatch instead of on the full training set, similar to momentum-like training procedures. It further allows for non-monotonically-increasing precision. Note that we revisit data during training, and the precision matrix is updated for every revisit. Thus, the forgetting induced by $\alpha$ is a wanted behavior to avoid infinite precision in the limit of infinite training time. In practice, we set $\alpha$ equal to the learning rate, where the practical intuition is that when $\alpha$ is small, the network uncertainty decreases faster.} 

The overall training procedure is summarized in \cref{fig:training_procedure}. We initialize $q^0$ as a Gaussian with $\thetab_0=0$ and $\mathbf{H}_0=\mathbb{I}$. We provide more details on the model, the linearization, and iterative learning in \cref{sec:model_appendix}.

\begin{figure}[t]
    \centering
    \includegraphics[width=0.8\textwidth]{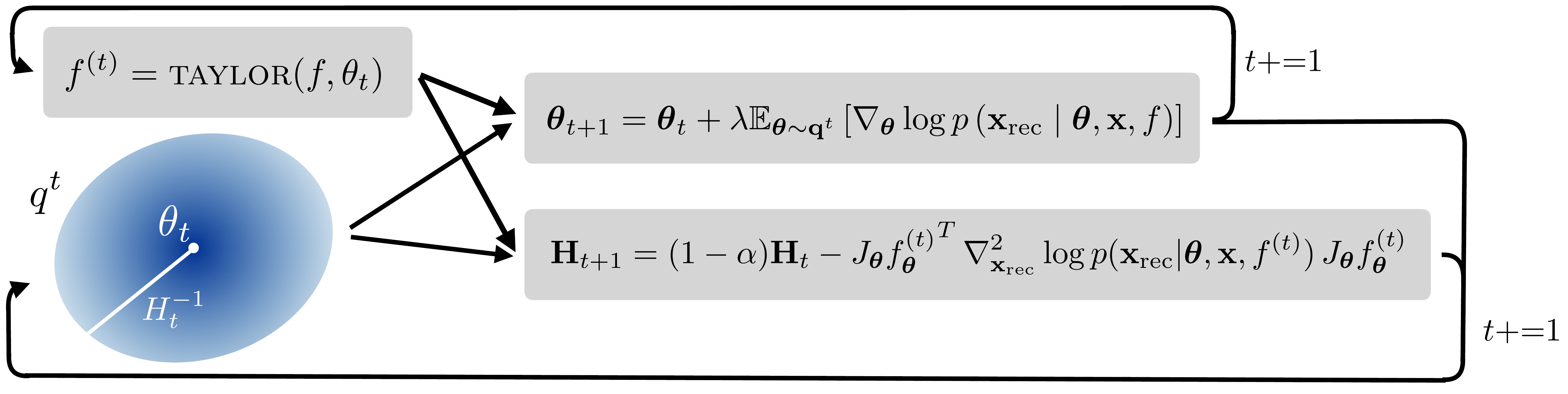}
    \caption{\textbf{\update{Iterative} training procedure.} Given a distribution $\q^t$ over parameters, and a linearized function $f^{(t)}$, compute first and second-order derivatives to update the distribution on parameters.}
    \label{fig:training_procedure}
\end{figure}

\textbf{Why not just\ldots?}~~~The proposed training procedure may at first appear non-trivial, and it is reasonable to wonder if existing methods could be applied to similar results. Variational inference often achieves similar results to Laplace approximations, so could we use `Bayes by Backprop' \citep{blundell2015bayesbybackprop} to get an alternative Gaussian approximate posterior over $\thetab$? Similar to the supervised experiences of \citet{valentin2020bnn}, we, unfortunately, found this approach too brittle to allow for practical model fitting. But then perhaps a post-hoc \textsc{la} as proposed by \citet{daxberger2021laplace} for supervised learning? Empirically, we found it to be important to center the approximate posterior around a point where the Hessian provides useful uncertainty estimates. Our online training moves in this direction as the Hessian is part of the procedure, but this is not true for the post-hoc \textsc{la}.

Conceptually, we argue that our approach, while novel, is not entirely separate from existing methods. Our reliance on lower bounds \update{makes} the method an instance of variational inference \citep{jordan1999introduction,opper2009variational}, and we maximize the bounds using Monte Carlo EM \citep{cappe2009online}. We rely on a \textsc{la} as our choice of variational distribution, which has also been explored by \citet{park2019variational}. Finally, we note that our linearization trick \eqref{eq:linearize_in_theta} has great similarities to classic extended Kalman filtering \citep{gelb1974applied}.

\section{Scaling the Hessian to Large Images}\label{sec:hessian}

The largest obstacle to apply \textsc{la} in practice \update{stems} from the Hessian matrix. This matrix has a quadratic memory complexity in the number of network parameters, which very quickly exceeds the capabilities of available hardware. 
To counter this issue, several approximations have been proposed \citep{ritter2018scalable, botev2017practical, Martens2015Kron} that improve the scaling w.r.t.\@ to the number of parameters.
The currently most efficient Hessian implementations \citep{dangel2020backpack, daxberger2021laplace} builds on the generalized Gauss-Newton (\textsc{ggn}) approximation of the Hessian
\begin{equation}
  \nabla^2_{\thetab^{(l)}} \mathcal{L}(f_\thetab(\xc))
  \approx
  J_{\thetab^{(l)}}f_\thetab\left(\xc\right)^{\!\top} \cdot
    \nabla^2_{\xc_\text{rec}}\mathcal{L}\left(\xc_\text{rec}\right) \cdot
    J_{\thetab^{(l)}}f_\thetab\left(\xc\right),
\end{equation}

\update{for a single layer $l$, which} neglects second order derivatives of $f$ w.r.t. the parameters. Besides, the computational benefits of this approximations, previous works on \textsc{la} \citep{daxberger2021laplace} relies on \textsc{ggn} to ensure that the Hessian is always semi-negative definite. In contrast, the model presented in \cref{sec:method} implies that \textsc{ggn} is no longer a practical and unprincipled trick, but rather the exact Hessian\update{~\citep{immer2021locallinerarization}}. 

Albeit relying on first order derivates, the layer-block-diagonal \textsc{ggn}, which \update{assumes} that layers are independent of each other, scales quadratically with the \emph{output} dimension of the considered neural network $f$. This lack of scaling is particularly detrimental for convolutional layers as these have low parameter counts, but potentially very high output dimensions.

Expanding $J_{\thetab^{(l)}}f_\thetab(\xc)$ with the chain rule, one realizes that the Jacobian can be computed as a function of the Jacobian of the next layer. \cref{fig:backprop_hessian} illustrate that an intermediate quantity $M$, which is initialised as $\nabla^2_{\xc_\text{rec}}\mathcal{L}(\xc_\text{rec})$, can be efficiently backpropagated through multiplication with the Jacobian w.r.t. input of each layer.
This process leads to a \textbf{block diagonal} approximation of the Hessian as illustrated in \cref{fig:ggn_methods}(a).  
However, diagonal blocks are generally too large to store and invert. To combat this, each block can be further approximated by its \textbf{exact diagonal} \citep{daxberger2021laplace} as depicted in \cref{fig:ggn_methods}(b). This scales linearly w.r.t.\@ parameters, but still scales quadratically w.r.t.\@ the output resolution (\cref{tab:computational_cost}).

\begin{figure}
  \begin{minipage}[c]{0.50\textwidth}
    \begin{tabular}{l|ll}
    \toprule
        \textsc{approximations} & \textsc{memory} & \textsc{time} \\ \midrule
        Block diag. 
          & $\mathcal{O}(R_m^2 + W_s^2)$
          & $\mathcal{O}(R_s^2 + W_s^2)$ \\
        KFAC 
          & $\mathcal{O}(R_s^2\, + W_s)$
          & $\mathcal{O}(R_s^2 + W_s)$ \\
        Exact diag. 
          & $\mathcal{O}(R_m^2 + W_s)$
          & $\mathcal{O}(R_s^2 + W_s)$ \\
        Approx.\@ diag. (ours) 
          & $\mathcal{O}(R_m + W_s)$
          & $\mathcal{O}(R_s + W_s)$ \\
        Mixed diag. (ours)
          & $\mathcal{O}(R_m + W_s)$
          & $\mathcal{O}(R_s + W_s)$ \\
          \bottomrule
    \end{tabular}
  \end{minipage}\hfill
  \begin{minipage}[c]{0.40\textwidth}
    \hspace{0.1mm}
    \captionof{table}{\textbf{Memory \& time complexity of Hessian approximations}. For an $L$-layer network, let $R_m$=$\max_{l=0\dots L}|x^{(l)}|$, $R_s$=$\sum_{l}|x^{(l)}|$, $R_s^2$=$\sum_{l}|x^{(l)}|^2$, and $W_s$=$\sum_l|\thetab^{(l)}|$. Only our approximation scales linearly with both the output resolution and parameters.}
    \label{tab:computational_cost}
  \end{minipage}
\end{figure}

\begin{figure}
    \centering
    \includegraphics[width=\textwidth]{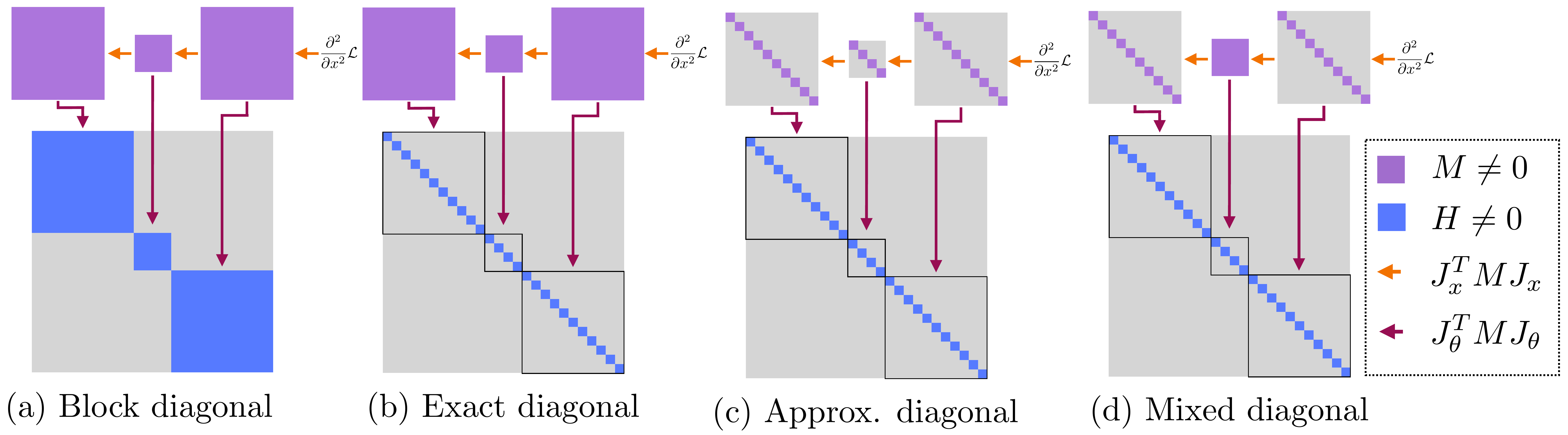}
    \vspace{-3mm}
    \caption{\textbf{Comparison of Hessian approximation methods.} Common approximations (a--b) scale quadratically with the output resolution. Our proposed approximate and mixed diagonal Hessians (c--d) scale linearly with the resolution. This is essential for scaling the \textsc{lae} to large images.}
    \label{fig:ggn_methods}
\end{figure}

\begin{wrapfigure}[14]{r}{0.4\textwidth} 
	\centering
	\vspace{-0.9cm}%
	\includegraphics[width=0.3\textwidth]{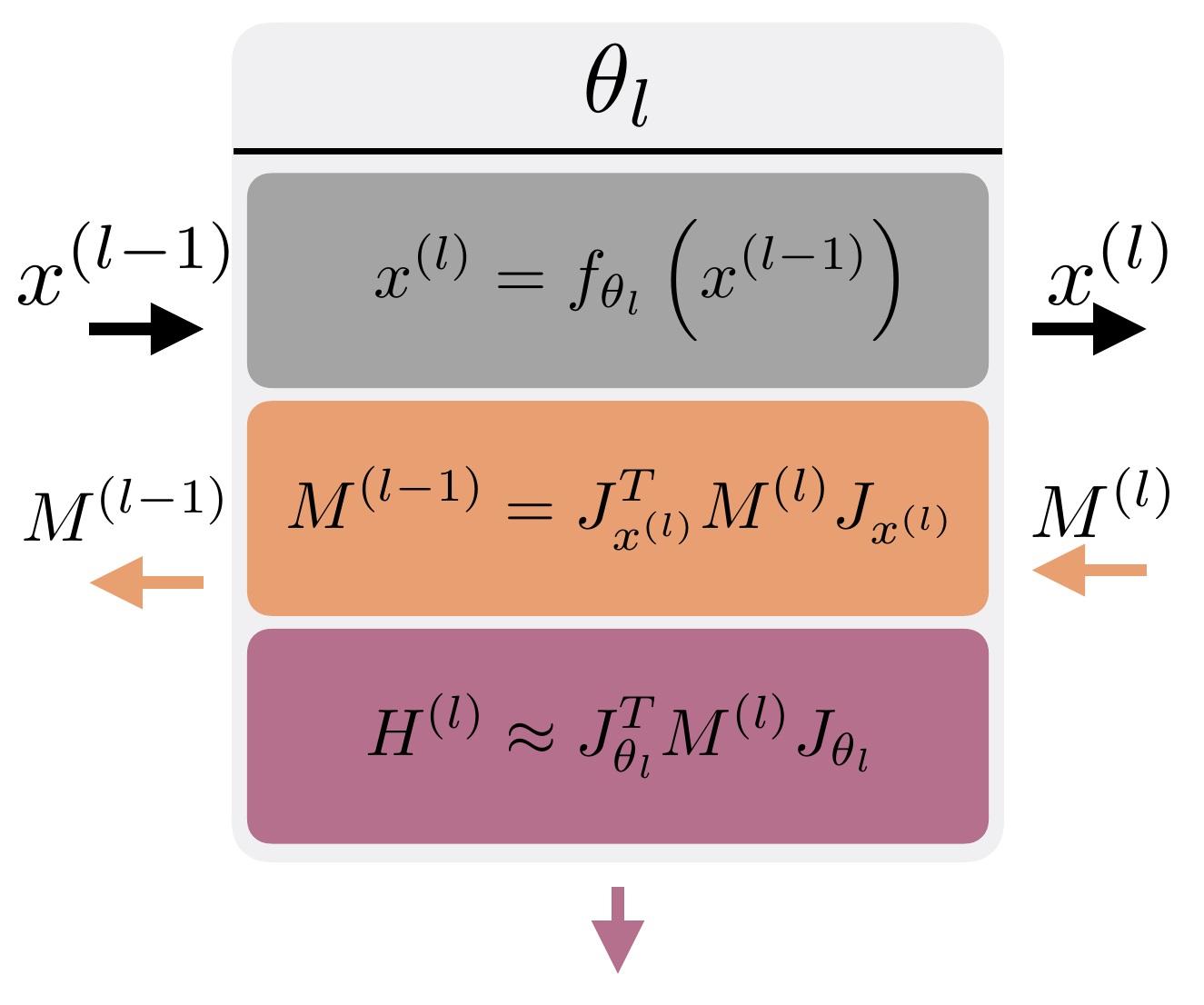}
	\vspace{-3mm}    
    \caption{Forward pass of feature map $x$ for layer $l$ with parameters $\thetab^{(l)}$ and \textcolor{orange}{extended backward pass in which $M$ is backpropagated to previous layers}. \textcolor{violet}{Via the chain rule and $M$ the Hessian of each layer can be computed efficiently.}}
    \label{fig:backprop_hessian}
\end{wrapfigure}

To scale our Laplacian autoencoder to high-dimensional data, we propose to approximate the diagonal of the Hessian rather than relying on exact computations. This is achieved by only backpropagating a diagonal form of $M$ as illustrated in \cref{fig:ggn_methods}(c). This assumes that features from the same layer are uncorrelated and consequently \update{have} linear complexity in both time and memory with respect to the output dimension (\cref{tab:computational_cost}). This makes it viable for our model.\looseness=-1

We can further tailor this approximation to the autoencoder setting by leveraging the bottleneck architecture. We note that the quadratic scaling of the exact diagonal Hessian is less of an issue in the layers near the bottleneck than in the layers closer to the output space. We can therefore dynamically switch between our approximate diagonal and the exact one, depending on the feature dimension. This lessens the approximation error while remaining tractable in practice. We provide more details on the fast hessian computations in \cref{sec:fast_hessian_appendix}.

\section{Related Work}
Deep generative models, and particularly the family of variational auto-encoders (\textsc{vae}s) \citep{kingma2014vae, rezende2014vae}, address unsupervised representation learning from a probabilistic viewpoint by approximating the posterior over the representation space. Despite their widespread adoption, model parameters are still \emph{deterministic} and sensitive to ill-suited local minima, e.g.\@ over-fitted to training data \citep{zhang21generalization}, which may cause poor generalization. The Bayesian \textsc{nn}s favour inference over the \textsc{nn} weights for addressing such issues \citep{mackay1995probable,neal1996bayesian}. This approach \update{deduces} distributions on data space by learning posterior distributions on the parameter space. However, several shortcomings \citep{wenzel2020good}, e.g.\@ expensive training, tuning, and implementations, often limit their applicability to autoencoder-style models. Alternatively, other methods such as deep ensembles~\citep{lakshminarayanan2017ensembles}, stochastic weight averaging (\textsc{swag})~\citep{maddox2019swag} or Monte-Carlo dropout~\citep{yarin2016mcdropout} also promise Bayesian approximations to \textsc{nn} weight's posterior, but at the cost of increased training time, poor empirical performance or limited Bayesian interpretation. 

As demonstrated by \cite{daxberger2021laplace} \textsc{la} is a scalable and well-behaved alternative to Bayesian \textsc{nn}s if used \emph{post-hoc} to approximate the intractable posterior over the weights after \emph{maximum-a-posteriori} (\textsc{map}) training in classification and regression. The general utility of \textsc{la} has also motivated its use as an approximation to the marginal likelihood over \textsc{nn} weights. Recent methods, including \citet{daxberger2021laplace}, have explored this path to find model hyperparameters \citep{immer2021scalable} or learning invariances \citep{immer2022invariance}. However, its computational burden, for instance in Hessian matrices, has prescribed diagonal or Kronecker factored approximations \citep{ritter2018scalable,martens2015optimizing,botev2017practical}, which are now widely used for second-order optimization. We provide more details on the connections to existing hessian based methods~\cite{daxberger2021laplace, Zhang2017NoisyNaturalGradient, Khan2017Variationaladaptivenewton}, Bayes by Backpropagation~\cite{blundell2015bayesbybackprop} and Adam~\cite{KingmaB14adam} in \cref{sec:extended_related_works}.

The \emph{full} Bayesian perspective on \textsc{vae} weights was first explored by \citet{daxberger2019bayesian} which we find similar in spirit to our work. In contrast to them (1) we follow a principled Bayesian derivation. (2) Neither do we depend on Hamiltonian Monte-Carlo sampling, which is generally hard to scale to efficient training.

\begin{figure}
    \centering
    \includegraphics[width=0.4\textwidth]{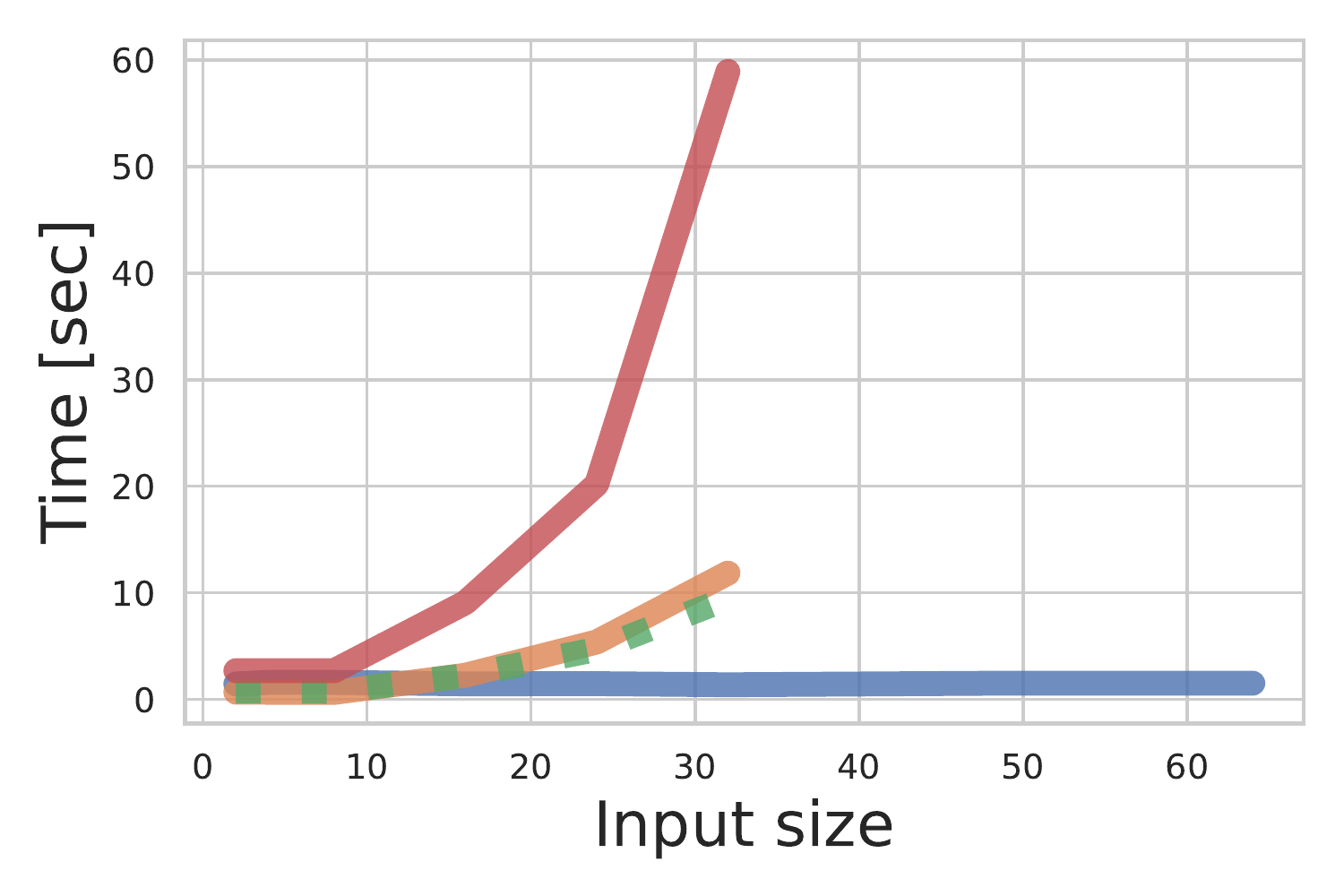}
    \includegraphics[width=0.4\textwidth]{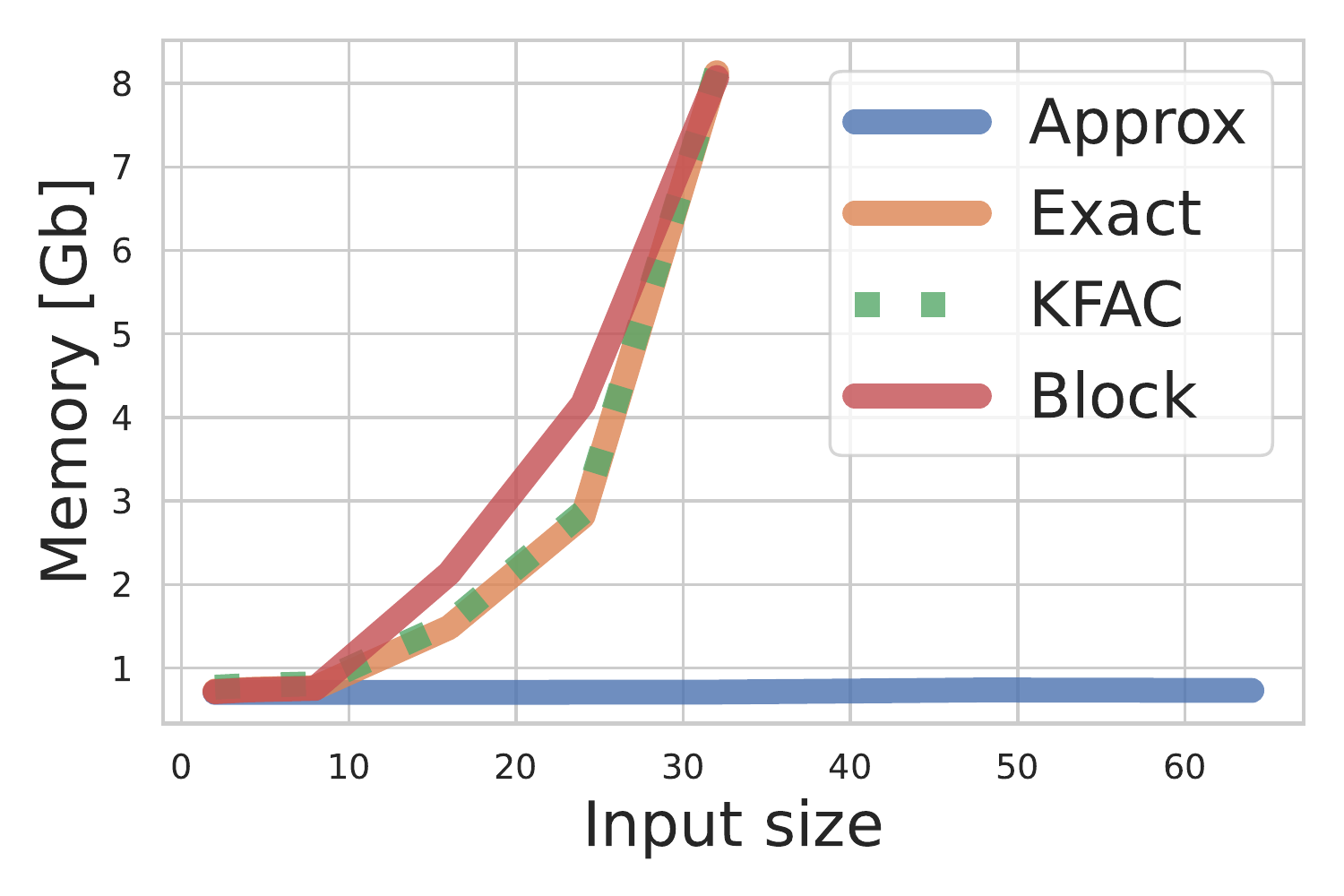}
\captionof{figure}{\textbf{Memory \& time usage of Hessian approximations.} The exact and \textsc{kfac} scales poorly with the image resolution. In contrast, our proposed approximate diagonal Hessian scales linearly.}
    \label{fig:computational_comparison}
\end{figure}

\section{Experiments}
First, we demonstrate the computational advantages of the proposed Hessian approximation, and then that our sampling-based training leads to well-calibrated uncertainties that can be used for \textsc{ood} detection, data imputation, and semi-supervised learning. For all the downstream tasks we consider the following baselines: \textsc{ae} \citep{hinton2006reducing} with constant and learned variance, \textsc{vae} \citep{rezende2014vae, kingma2014vae}, Monte-Carlo dropout \textsc{ae} \citep{yarin2016mcdropout} and Ensembles of \textsc{ae} \citep{lakshminarayanan2017ensembles}. We extend StochMan\cite{software:stochman} with the Hessian backpropagation for the approximate and mixed diagonals. The training code is implemented in PyTorch and available$^2$\footnote{$2$ \url{https://github.com/FrederikWarburg/LaplaceAE}}. \cref{sec:experiment_appendix} provides more details on the experimental setup.

\textbf{Efficient Hessian Approximation.} For practical applications, training time and memory usage of the Hessian approximation must be kept low. We here show that the proposed approximate diagonal Hessian is sufficient and even outperforms other approximations when combined with our online training.\looseness=-1

\cref{fig:computational_comparison} show the time and memory requirement for different approximation methods as a function of input size for a 5-layer convolutional network that \update{preserves} channel and input dimension. As baselines we use efficient implementations of the exact and \textsc{kfac} approximation~\citep{daxberger2021laplace, dangel2020backpack}. The exact diagonal approximation run out of memory for an $\sim 36 \times 36 \times 3$ image on a $11$ Gb NVIDIA GeForce GTX 1080 Ti. In contrast, our approximate diagonal Hessian scales linearly with the resolution, which is especially beneficial for convolutional layers. 


\begin{minipage}[ht!]{\textwidth}
\begin{minipage}{0.5\textwidth}

    \includegraphics[width=\textwidth]{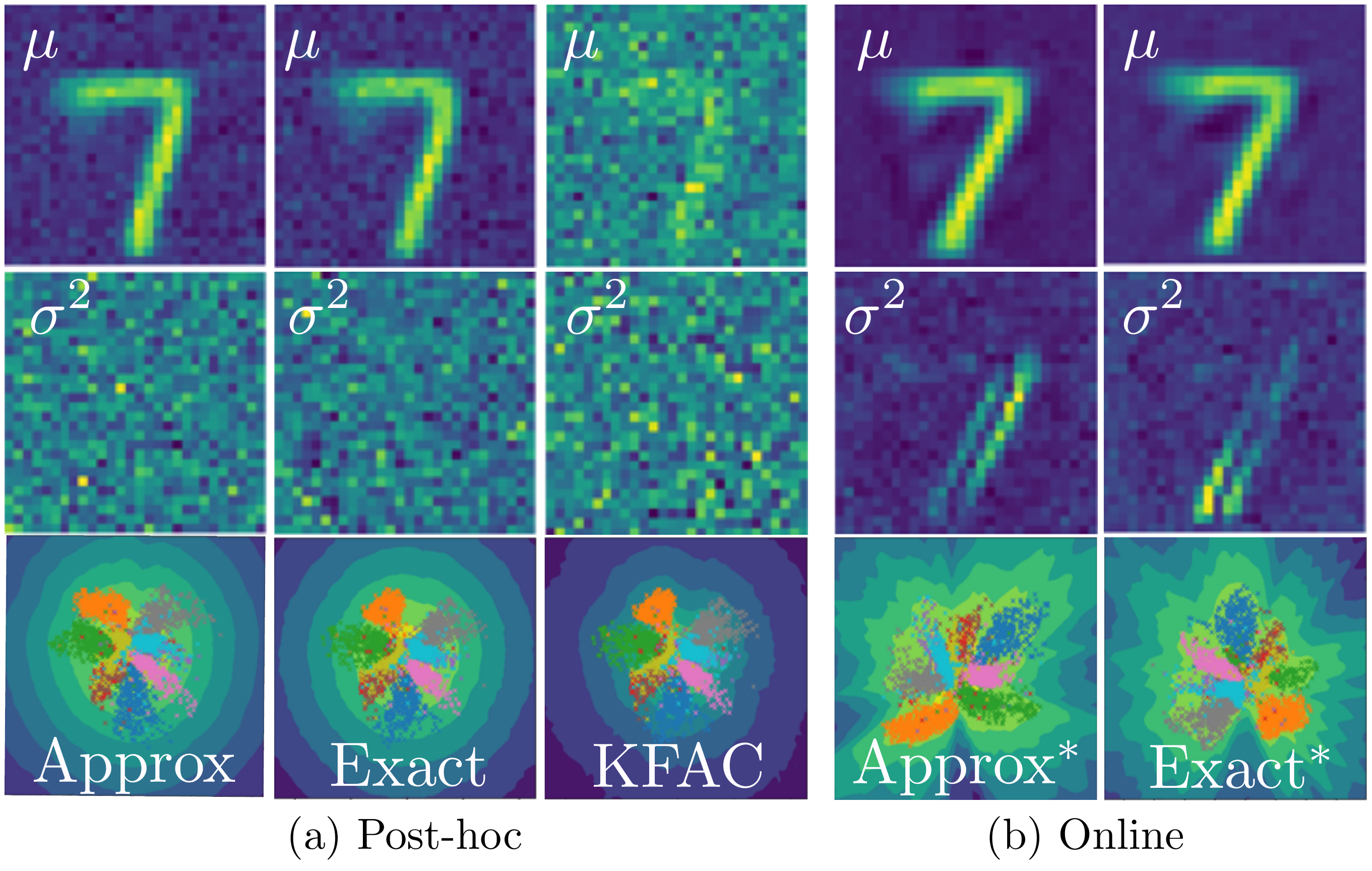}
    \captionof{figure}{Mean and variance of $100$ sampled \textsc{nn}.\looseness=-1 }
    \label{fig:hessian_approx_examples}
\end{minipage}
\hspace{3mm}
\begin{minipage}{0.47\textwidth}
\centering
    \update{\resizebox{\textwidth}{!}{\begin{tabular}{l|ll}
    \toprule
        Hessian &  $- \log p(x) \downarrow$ & MSE $\downarrow$\\ \midrule
        
        \textsc{kfac} & $9683.9 \pm 2455.0$ & $121.6 \pm 24.5$\\
        Exact & $283.3 \pm 88.6$ & $27.1 \pm 0.9$\\
        Approx & $232.0 \pm 65.5$ & $26.6 \pm 0.6$\\
        Exact$^*$ & $25.8 \pm 0.2$ & $25.7 \pm 0.2$ \\
        Approx$^*$ & $25.9 \pm 0.4$ & $25.8 \pm 0.4$ \\
    \bottomrule
    \end{tabular}}}
\captionof{table}{Online training (indicated by $^*$) outperforms post-hoc \textsc{la}. The approximate diagonal has similar performance to the exact diagonal for both post-hoc and online \textsc{la}.}
\label{tab:approx_quality}
\end{minipage}
\end{minipage}

\cref{tab:approx_quality} shows that the exact or approximate Hessian diagonal has similar performance \update{for both post-hoc and online training}. Using post-hoc \textsc{la} results in good mean reconstructions (low MSE), but each sampled \textsc{nn} does not give good reconstructions (low $\log p(x)$). Using our online training procedure results in a much higher log-likelihood. This \update{indicates} that every sampled \textsc{nn} \update{predicts} good reconstructions. 

\cref{fig:hessian_approx_examples} shows the latent representation, mean, and variance of the reconstructions with the \textsc{kfac}, exact and approximate diagonal for both post-hoc and online setup. Note that the online training makes the uncertainties better fitted, both in latent and data space. These well-fitted uncertainties have several practical downstream applications, which we demonstrate next. 

\textbf{Out-of-Distribution (\textsc{OoD}) Detection} capabilities are critical for identifying distributional shifts, outliers, and irregular user inputs, which can hinder the propagation of erroneous decisions in an automated system. We evaluate \textsc{OoD} performance on the commonly used benchmarks \citep{nilisnick2019knownothing}, where we use \textsc{FashionMnist} \citep{xiao2017fashionmnist} as in-distribution and \textsc{mnist} \citep{lecun1998mnist} as \textsc{OoD}. 
\cref{fig:roccurves} (c) shows that our online \textsc{lae} outperforms existing models in both log-likelihood and Typicality score \citep{nalisnick2019typicality}. This stems from the calibrated model uncertainties, which are exemplified in the models ability to detect \textsc{OoD} examples from the uncertainty deduced in latent and output space; see \cref{fig:roccurves} (a,b) for \textsc{roc} curves.

\cref{fig:ood_histograms} shows distribution of the output variances for in- and \textsc{OoD} data. This illustrates that using \textsc{la} improves OoD detection. Furthermore, the online training improves the model calibration.

\begin{figure}[h]
    \centering
    \includegraphics[width=1\textwidth]{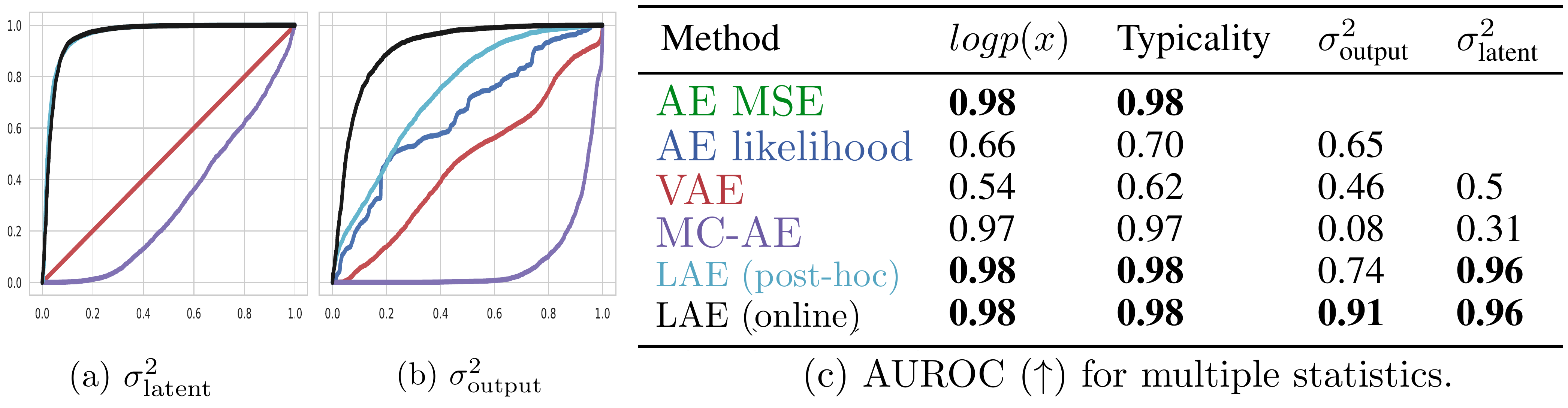}
    \caption{\textbf{Out of Distribution detection.} In-distribution data FashionMnist and \textsc{OoD} data Mnist. (a) and (b) shows the ROC curves for latent and output space uncertainties. (c) shows (AUROC $\uparrow$) for log-likelihood, typicality score, latent space $\sigma_{\text{latent}}$ and output space $\sigma_{\text{output}}$ uncertainties. Online \textsc{lae} is able to discriminate between in and \textsc{OoD} using the deduced variances in latent and output space.}
    \vspace{-0.5cm}
    \label{fig:roccurves}
\end{figure}

\begin{figure}[b]
    \centering
    \includegraphics[width=0.8\textwidth]{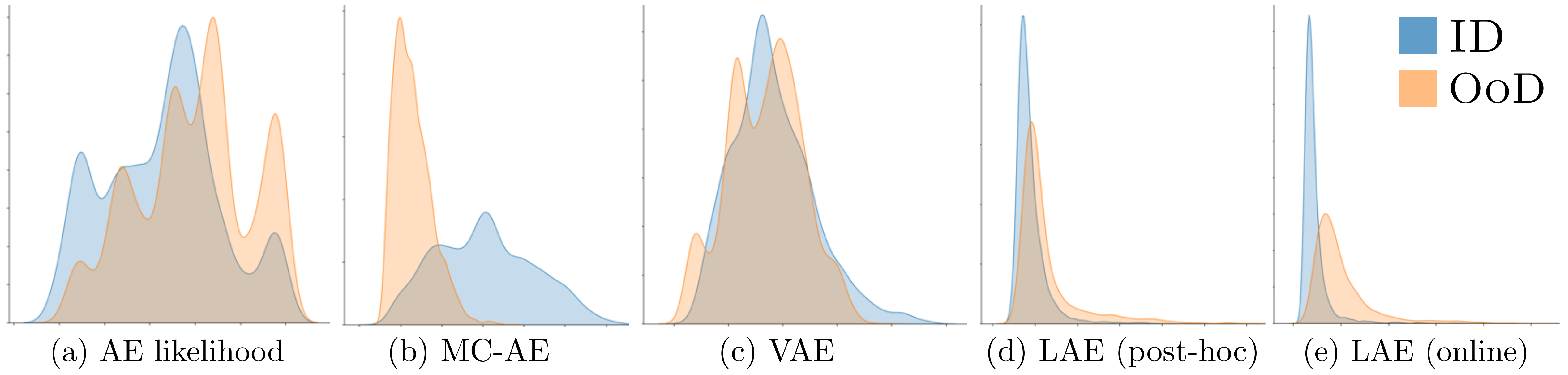}
    \caption{\textbf{Histograms of variance} for in- and \textsc{OoD} reconstructions in output space. Note that MC-AE separates the distributions well, but the model \update{assigns} higher variance to the in- than \textsc{OoD} data.\looseness=-1}
    \vspace{-0.6cm}
    \label{fig:ood_histograms}
\end{figure}

\textbf{Missing Data imputation.} Another application of stochastic representation learning is to provide distributions over unobserved values \citep{rezende2014vae}. In many application domains, sensor readings go missing, which we may mimic by letting parts of an image be unobserved. \citet{rezende2014vae} show that we can then draw samples from the distribution of the entire image conditioned on the observed part, by imputing the missing pixels with noise and repeatedly encode and decode while keeping observed pixels fixed. \cref{fig:data_imputation} show samples using this procedure from a \textsc{vae}, a post-hoc \textsc{lae} and our online \textsc{lae}, where we only observe the lower half of an \textsc{mnist} image. This implies ambiguity about the original digit, e.g.\@ the lower half of a ``5`` could be a ``3`` and similarly a ``7`` could be a ``9``. Our \textsc{lae} captures this ambiguity, which is exemplified by the multi-modal reconstructions from the sampled networks in \cref{fig:data_imputation}. The baselines only capture unimodal reconstructions.

Capturing the ambiguity of partly missing data can improve downstream tasks such as the calibration of an image classifier. In \cref{fig:data_imputation} (c) we demonstrate how averaging the predictions of a simple classifier across reconstructions improves standard calibration metrics. This is because the classifier inherits the uncertainty and ambiguity in the learned representations. A deep ensemble of \textsc{ae}s \update{performs} similarly to ours, but comes at the cost of training and storing multiple models. 

When the entire input image is missing, the imputation procedure can be seen as a sampling mechanism, such that our \textsc{lae} can be viewed as a generative model. The bottom rows in \cref{fig:data_imputation} show that the \textsc{lae} indeed does generate sharp images from a multi-modal distribution.

\begin{figure}
    \centering
    \includegraphics[width=\textwidth]{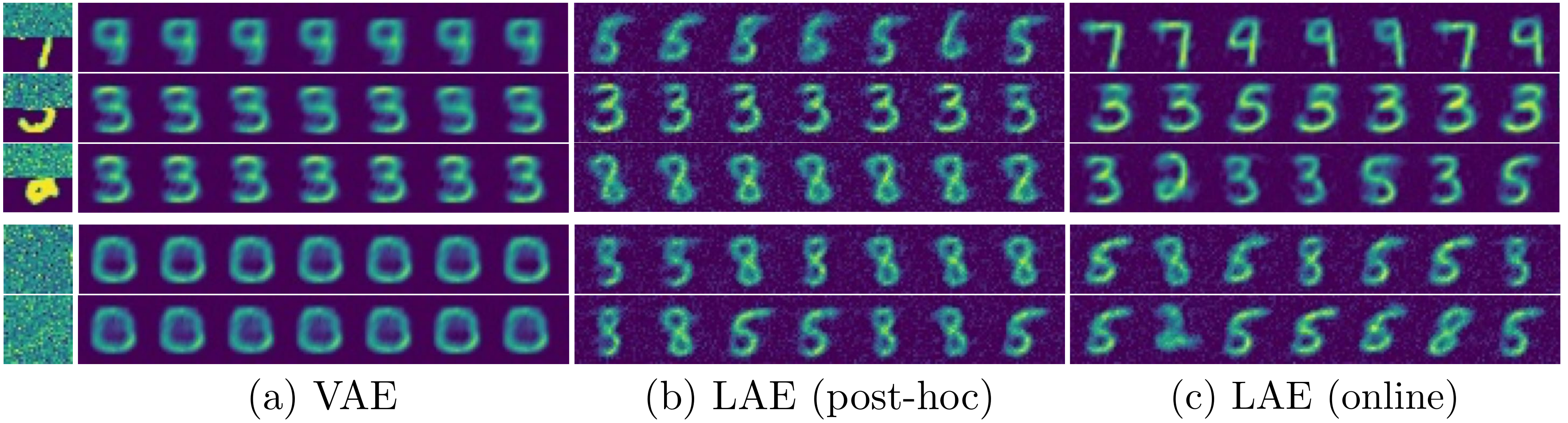}
    \vspace{-3mm}
    \caption{\textbf{Missing data imputation \& generative capabilities.} Online training of \textsc{lae} improves representational robustness. This is exemplified by the multimodal behavior in the data imputation (top rows) that accurately model the ambiguity in the data. The bottom two rows \update{show} that the \textsc{lae} is able to generate crisp digits from random noise.}
    \label{fig:data_imputation}
\end{figure}

\begin{table}[]
    \centering
    \begin{tabular}{l|llllll}
    \toprule 
        Method          &  MSE $\downarrow$  &$\log p(x) \uparrow$ & Acc. $\uparrow$ & ECE $\downarrow$  & MCE  $\downarrow$ &   RMSCE $\downarrow$ \\ \midrule
        Classifier      &       &           &  \bf{0.53} &  0.16 & 0.25 & 0.18     \\
        VAE             & 104.73 & -104.75   & 0.22  &  0.18 &  0.34 &   0.19  \\
        MC-AE           &  106.05 & -106.05 & 0.45       &  0.28 & 0.38 & 0.29     \\
        AE Ensemble     &  \bf{96.23} & \bf{-100.94} & \bf{0.53}      &  \bf{0.12} & 0.2 & \bf{0.13}         \\
        LAE (post-hoc)     & 101.62 & -107.25 & 0.51    &  0.16 & 0.27  &    0.18 \\
        LAE (online)      & 99.59 & -106.29    & \bf{0.53}  & \bf{0.12}& \bf{0.16} & \bf{0.13} \\
    \bottomrule
    \end{tabular}
    \caption{Reconstruction quality measured by the MSE and log-likelihood for the data imputation. Our well-calibrated uncertainties propagates to the \textsc{mnist} classifier and improves the calibration metrics ECE, MCE and RMSCE.\looseness=-1}
    \label{tab:data_imputation}
    \vspace{-0.8cm}
\end{table}

\begin{minipage}[H]{\textwidth}
\begin{minipage}{0.45\textwidth}
    \includegraphics[width=1\textwidth]{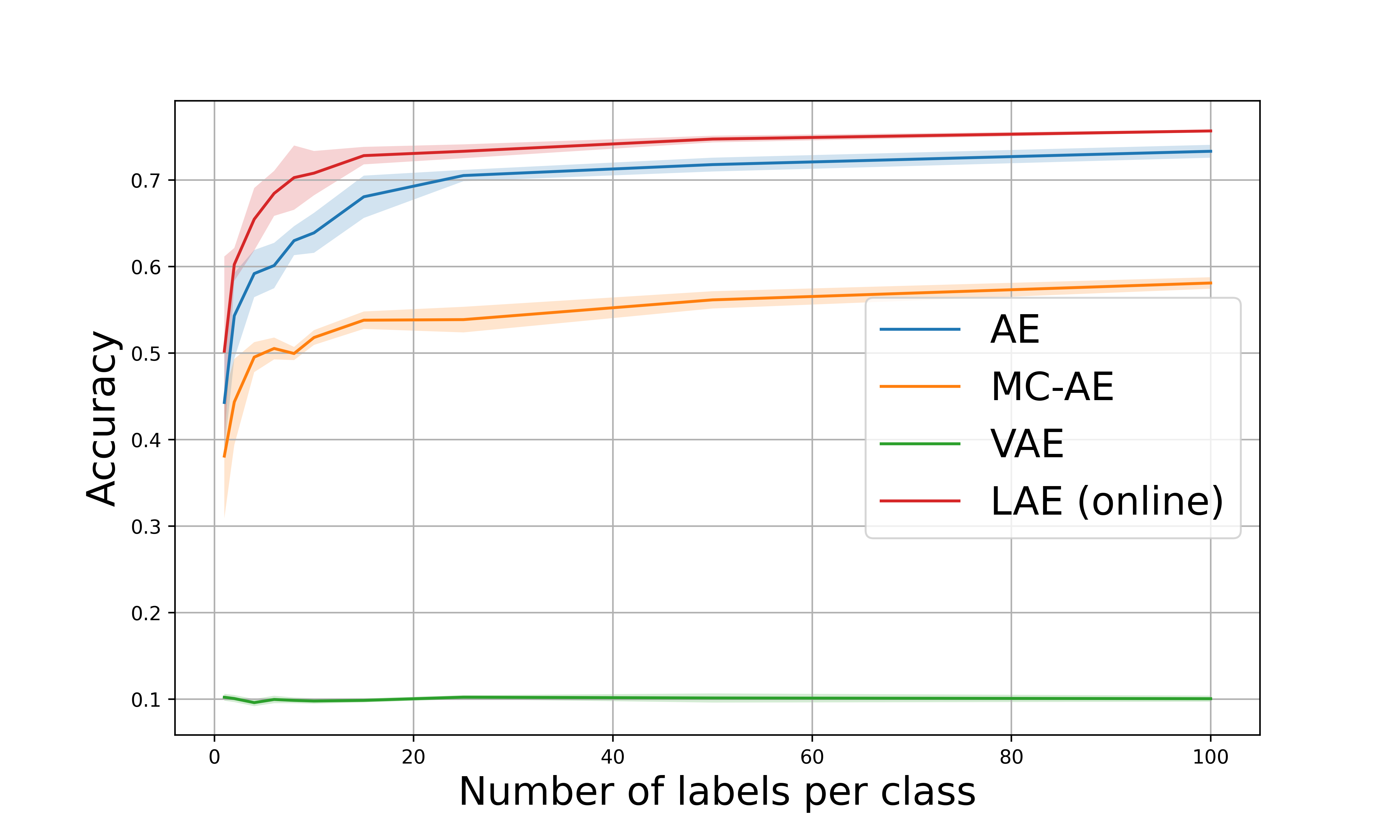}
    \captionof{figure}{Accuracy as an function of the number of labels per class on \textsc{mnist}\looseness=-1. 
    }
    \label{fig:semi_mnist}
\end{minipage}
\hspace{5mm}
\begin{minipage}{0.50\textwidth}
    \resizebox{\textwidth}{!}{\begin{tabular}{l|llll}
    \toprule
        Attribute & AE & VAE & MC-AE & LAE$^*$ \\ \midrule
        Arched Eyebrows  & 0.50 & 0.52 & 0.55 & 0.60\\
        Attractive       & 0.52 & 0.50 & 0.49 & 0.53 \\  
        Bald             & 0.98 & 0.98 & 0.98 & 0.98 \\
        Wearing Lipstick & 0.52 & 0.49 & 0.50 & 0.54 \\ 
        Heavy Makeup     & 0.45 & 0.52 & 0.49 & 0.56 \\ \midrule
        Overall          & 0.73 & 0.72 & 0.73 & \textbf{0.74} \\
    \bottomrule
    \end{tabular}}
    \captionof{table}{Semi-supervised classification accuracy on \textsc{CelebA} using only $10$ labeled datapoints. $^*$ refers to online LAE.}
    \label{tab:semi_celeba}
\end{minipage}
\end{minipage}

\begin{figure}
     \centering
     \begin{subfigure}[b]{0.19\textwidth}
         \centering
         \includegraphics[width=\textwidth]{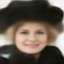}
     \end{subfigure}
     \begin{subfigure}[b]{0.19\textwidth}
         \centering
         \includegraphics[width=\textwidth]{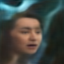}
     \end{subfigure}
     \begin{subfigure}[b]{0.19\textwidth}
         \centering
         \includegraphics[width=\textwidth]{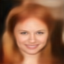}
     \end{subfigure}
     \begin{subfigure}[b]{0.19\textwidth}
         \centering
         \includegraphics[width=\textwidth]{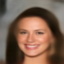}
     \end{subfigure}
     \begin{subfigure}[b]{0.19\textwidth}
         \centering
         \includegraphics[width=\textwidth]{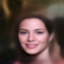}
     \end{subfigure}
     
     \begin{subfigure}[b]{0.19\textwidth}
         \centering
         \includegraphics[width=\textwidth]{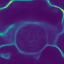}
     \end{subfigure}
     \begin{subfigure}[b]{0.19\textwidth}
         \centering
         \includegraphics[width=\textwidth]{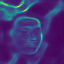}
     \end{subfigure}
     \begin{subfigure}[b]{0.19\textwidth}
         \centering
         \includegraphics[width=\textwidth]{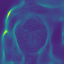}
     \end{subfigure}
     \begin{subfigure}[b]{0.19\textwidth}
         \centering
         \includegraphics[width=\textwidth]{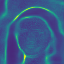}
     \end{subfigure}
     \begin{subfigure}[b]{0.19\textwidth}
         \centering
         \includegraphics[width=\textwidth]{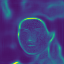}
     \end{subfigure}
    \caption{\update{\textbf{Sample reconstructions on \textsc{CelebA}.} The top row shows the mean reconstruction and the bottom row shows the variance of the reconstructed images.}}
    \vspace{-0.5cm}
    \label{fig:celeba_samples}
\end{figure}

\textbf{Semi-supervised learning} combines a small amount of label data with a large amount of unlabeled data. The hope is that the  structure in the unlabeled data can be used to infer properties of the data that cannot be extracted from a few labeled points. Embedding the same labeled data point multiple times using a stochastic representation scales up the amount of labeled data that is available during training.\looseness=-1

\cref{fig:semi_mnist} shows the accuracy of a $K$-nearest neighbor classifier trained on different amounts of labeled data from the \textsc{mnist} dataset. For all models with a stochastic encoder, we encode each labeled datapoint $100$ times and repeat the experiment $5$ times. When only a few labels per class are available ($1$-$20$) we clearly observe that our \textsc{lae} model outperforms all other models, stochastic and deterministic. Increasing the number of labels beyond $100$ per class makes the \textsc{ae} and \textsc{lae} equal in their classification performance with the \textsc{ae} model eventually outperforming the \textsc{lae} model. 

In \cref{tab:semi_celeba} we conduct a similar experiment on the \textsc{CelebA} \citep{liu2015faceattributes} facial dataset, where the the task is to predict $40$ different binary labels per data point. When evaluating the overall accuracy of predicting all $40$ facial attributes, we see no significant difference in performance. However, when we zoom in on specific facial attributes we gain a clear performance advantage over other models. \update{\cref{fig:celeba_samples} shows the mean and variance of five reconstructed images. The online \textsc{lae} produces well-calibrated uncertainties in the output space and scales to large images.}

\textbf{Limitations.}~~~Empirically, the \textsc{lae} improvements are more significant for overparameterized networks. The additional capacity seems to help the optimizer find a local mode where a Gaussian fit is appropriate. It seems the regularization induced by marginalizing $\thetab$ compensates for the added flexibility.\looseness=-1 
\section{Conclusion}
In this paper, we have introduced a Bayesian autoencoder that is realized using Laplace approximations. Unlike current models, this Laplacian autoencoder produces well-behaved uncertainties in both latent and data space. We have proposed a novel variational lower-bound of the autoencoder evidence and an efficient way to compute its Hessian on high dimensional data that scales linearly with data size. Empirically, we demonstrate that our proposed model predicts reliable stochastic representations that are useful for a multitude of downstream tasks: out-of-distribution detection, missing data imputation, and semi-supervised classification. 
Our work opens the way for fully Bayesian representation learning where we can marginalize the representation in downstream tasks. We find this to consistently improve performance.\looseness=-1 

\begin{ack}
This work was supported by research grants (15334, 42062) from VILLUM FONDEN. This project has also received funding from the European Research Council (ERC) under the European Union's Horizon 2020 research and innovation programme (grant agreement 757360). This work was funded in part by the Novo Nordisk Foundation through the Center for Basic Machine Learning Research in Life Science (NNF20OC0062606).
\end{ack}

\newpage
\appendix
\begin{center}
    \Large{\textbf{--- Supplementary Material ---}}
\end{center}
The supplementary material is organized as follows. First, we give more technical details on the experiments. Second, we discuss the Laplace approximation and the ``mean shift'' issue encountered in non-local maxima. Third, we elaborate more on the difference between 
related works and our proposed method. Fourth, we present a much more thorough explanation of the proposed method accompanied by relevant proofs. Fifth, we present more details on the hessian derivations.

\section{Experimental details}~\label{sec:experiment_appendix}
All experiments was conducted on one of the following datasets: \textsc{mnist} \citep{lecun1998mnist}, \textsc{fashionmnist} \citep{xiao2017fashionmnist} and \textsc{celeba} \citep{liu2015faceattributes}. For training and testing the default splits were used. For validation, we sampled 5000 data points randomly from the training sets. All images were normalized to be in the $[0,1]$ range and for the \textsc{celeba} dataset the images were resized to $64 \times 64$ additionally.

If nothing else is stated, the models were trained with the following configuration: we used Adam optimizer \citep{kingma2015adam} with a learning rate of 0.001 and default Pytorch settings \citep{paszke2019pytorch}. The learning rate was adjusted using \textit{ReduceLROnPlateau} learning rate scheduler with parameters \textit{factor=0.5} and \textit{patience=5}, meaning that the learning rate was halved whenever 5 epochs had passed where the validation loss had not decreased. The mean squared error loss was used as the reconstruction loss in all models. Models were trained until convergence, defined as whenever the validation loss had not decreased for 8 epochs. Models trained on \textsc{mnist} and \textsc{fashionmnist} used a batch size of 64 and on \textsc{celeba} a batch size of 128 was used.

Model-specific details:
\begin{itemize}
    \item \textsc{vae}: Models were trained with KL-scaling of 0.001. We use two encoders and two decoders, such that the model has twice the number of parameters compared to the other models.
    \item \textsc{mc-ae}: Models were trained with dropout between all trainable layers with probability $p=0.2$. We keep the same dropout rate during testing.
    \item \textsc{emsemble-ae}: Each ensemble consists of $5$ models, each initialized with a different random seed.
    \item \textsc{lae (posthoc)}: For experiments with linear layers we used the Laplace Redux~\citep{daxberger2021laplace} implementation. For convolutions, we found it necessary to use our proposed hessian approximation. We use a diagonal approximation of the hessian in all experiments. After fitting the hessian, we optimize for the prior precision using the marginal likelihood~\citep{daxberger2021laplace}. We use $100$ MC sampling in all experiments.   
    \item \textsc{lae (online)}: We use the exact diagonal in experiments with linear layers and the mixed diagonal approximation in all experiments with convolutional layers. We use a hessian memory factor of $0.0001$ and sample only $1$ network per iteration. We found that it was not necessary to optimize for the prior precision when trained online.
\end{itemize}

\subsection{Hessian approximation}
We use a linear encoder-decoder with three layers in the encoder and decoder, \textsc{tanh} activation functions, and latent size $2$. We choose this architecture as Laplace Redux~\citep{daxberger2021laplace} supports various hessian approximations for this simple network. We use Laplace Redux for all post-hoc experiments except for the approximate diagonal hessian. 

\subsection{Out-of-distribution}
We use a convolutional encoder-decoder architecture. The encoder consisted of a \textsc{conv2d}, \textsc{tanh}, \textsc{maxpool2d}, \textsc{conv2d}, \textsc{tanh}, \textsc{maxpool2d}, \textsc{linear}, \textsc{tanh}, \textsc{linear}, \textsc{tanh}, \textsc{linear}, where the decoder was mirrored but with nearest neighbour Upsampling rather than \textsc{maxpool2d}. We used a latent size of $2$ in these experiments for all models. 

\subsection{Missing data imputation}
To elaborate on the procedure, we reconstruct $5$ samples from the half/fully masked image. For each of these reconstructions, we make $5$ more reconstructions and take the average of these reconstructions. The intuition is that the first stage explores the multi-modal behavior of the reconstructions. In the second stage, the uncertainty of the reconstructed digit is reduced, and each sample will reconstruct the same modality. By averaging over these modalities, we achieve a more crisp reconstruction. We use the same architecture as in the hessian approximation experiment. 

\subsection{Semi-supervised learning}
For the experiments on \textsc{mnist}, we use a single convergence checkpoint for each model. We use the same model architecture as in the hessian approximation. We did 5 repetitions for each model where we first sampled $n$ labels from each of the $10$ classes from the validation set, then embedded the $10\times n$ data points into latent space, trained a \textsc{knn} classifier on the embedded points and finally evaluated the accuracy of the classifier on the remaining validation set. This procedure was repeated for different values of $n$ in the $[1,100]$ range. For the stochastic encoders (\textsc{vae}, \textsc{mc-ae}, \textsc{lae}), we repeated the embedding step $100$ times with the goal that the uncertainty could help the downstream classification. For the \textsc{knn} classifier we use cross-validation ($K=2$) to find the optimal number of nearest neighbors.

For the experiments on \textsc{celeba} we repeated the exact same experiments but with a fixed value if $n=10$. Additionally, the classifier was changed to a multi-label version \textsc{knn}-classifier to accommodate the multiple binary features in the dataset. For \textsc{celeba} we use a convolutional architecture. The encoder consists of $5$ convolutional layers with \textsc{tanh} and \textsc{maxpool2d} in between each parametric layer. We use a latent size of $64$. The decoder mirrors the encoder, but we replace \textsc{maxpool2d} with nearest neighbor upsampling. 

\section{Laplace Approximation}\label{sec:laplace_approx_appendix}
Laplace approximation is an operator that maps local properties (derivatives) to global properties. The idea is to infer a density on every point based on the curvature in a single point. This is done through a Taylor expansion.

Given a vectorial space $\Theta$ of size $D$, let $P(\thetab)$ be an arbitrary distribution on $\Theta$ and let $\thetab^*\in\Theta$ be an arbitrary point. Consider the second order Taylor expansion of the log density around $\thetab^*$
\[
\ln P(\thetab) = \ln P(\thetab^*) + \nabla_\thetab\ln P (\thetab^*)(\thetab-\thetab^*) + \frac{1}{2}(\thetab-\thetab^*)^\top \nabla^2_\thetab\ln P (\thetab^*)(\thetab-\thetab^*) + \mathcal{O}(\|\thetab-\thetab^*\|^3),
\]
where
\[
[\nabla_\thetab\ln P (\thetab^*)]_i = \frac{\partial}{\partial \thetab_i}
\ln P(\thetab)\bigg\rvert_{\thetab=\thetab^*}
\qquad
[\nabla^2_\thetab\ln P (\thetab^*)]_{ij} = \frac{\partial^2}{\partial \thetab_i \partial \thetab_j}
\ln P(\thetab)\bigg\rvert_{\thetab=\thetab^*}.
\]
are the first and second-order derivatives.

Note that if $P(\thetab)$ is a Gaussian, then its log density is a second order polynomial and the second order Taylor expansion is exact. This implies that starting from a Gaussian, the Laplace approximation can infer the full exact density just from the values of $\nabla_\thetab\ln P$ and $\nabla^2_\thetab\ln P$ in a single point $\thetab^*$.

The main takeaway from the Laplace approximation is the strong, intrinsic, tie between the covariance matrix and the negative inverse of the hessian of the log probability: $-(\nabla^2_\thetab\ln P (\thetab^*))^{-1}$.

\subsection{If $\thetab^*$ is a local maxima}
If $\nabla_\thetab\ln P(\thetab^*)=0$ the Taylor expansion consists of only two terms and the Laplace derivation is easier. We also present this in order to develop intuition, although this case is a subcase of the non-local maxima case. 
In the next section, we will consider the more general setting.

Define the Gaussian
\begin{equation}
    \textsc{laplace}_{\text{max}}(\thetab^*;P) 
    := \mathcal{N}
    \left(\thetab|
        \mu=\thetab^*,
        \sigma^2=-(\nabla^2_\thetab\ln P (\thetab^*))^{-1}
    \right),
\end{equation}
which has density
\[
Q(\thetab) := \frac{P(\thetab^*)}{Z_Q} e^{\frac{1}{2}(\thetab-\thetab^*)^\top (\nabla^2_\thetab\ln P (\thetab^*))^{-1}(\thetab-\thetab^*)},
\]
where $Z_Q=P(\thetab^*)\sqrt{ (-2\pi)^{D} \det(\nabla^2_\thetab\ln P (\thetab^*))}$ is the normalizing constant. Then, $Q(\thetab)$ is a good approximation of $P(\thetab)$ in the sense that
\begin{align*}
    \ln Q(\thetab) + \ln Z_Q
    & = \ln P(\thetab^*) + \frac{1}{2}(\thetab-\thetab^*)^\top (\nabla^2_\thetab\ln P (\thetab^*))^{-1}(\thetab-\thetab^*) \cong \\
    & \cong \ln P(\thetab^*) + \frac{1}{2}(\thetab-\thetab^*)^\top (\nabla^2_\thetab\ln P (\thetab^*))^{-1}(\thetab-\thetab^*)  + \mathcal{O}(\|\thetab-\thetab^*\|^3) = \\
    & = \ln P(\thetab).
\end{align*}
Notice that $\thetab^*$ is in a local maximum, which ensures that the hessian $\nabla^2_\thetab\ln P (\thetab^*)$ is negative semi-definite. This in turn ensures that the normalizing constant $Z_Q$ exists and that $\textsc{laplace}_{\text{max}}(\thetab^*;P)$ is well defined.

\subsection{If $\thetab^*$ is not a local maxima}
In order to proceed to a similar derivation when $\nabla_\thetab\ln P(\thetab^*)\not=0$, we first rearrange the terms in the Taylor expansion. For a more compact notation, we write $\nabla\ln P$ instead of $\nabla_\thetab\ln P(\thetab^*)$ and  $\nabla^2\ln P$ instead of $\nabla^2_\thetab\ln P(\thetab^*)$.
\begin{align*}
    \ln P(\thetab)
    & \cong \ln P(\thetab^*) + \nabla\ln P(\thetab-\thetab^*) + \frac{1}{2}(\thetab-\thetab^*)^\top \nabla^2\ln P(\thetab-\thetab^*) = \\
    & = \ln P(\thetab^*) - \frac{1}{2}\nabla\ln P^\top \nabla^2\ln P^{-1} \nabla\ln P \\ 
    & \qquad + \frac{1}{2}(\thetab-\thetab^* + \nabla^2\ln P^{-1} \nabla\ln P)^\top \nabla^2\ln P(\thetab-\thetab^* + \nabla^2\ln P^{-1} \nabla\ln P) = \\
    & = \ln P(\thetab^*) - \frac{1}{2}\nabla\ln P^\top \nabla^2\ln P^{-1} \nabla\ln P + \frac{1}{2}(\thetab-\thetab^*_1)^\top \nabla^2\ln P(\thetab-\thetab^*_1 ),
\end{align*}
where we define the new point $\thetab^*_1$ as 
\begin{equation}
    \thetab^*_1 
    = \thetab^* - (\nabla^2_\thetab\ln P(\thetab^*))^{-1} \nabla_\thetab\ln P(\thetab^*).
\end{equation}
Define the Gaussian
\begin{equation}\label{eq:laplace_definition}
    \textsc{laplace}(\thetab^*;P) 
    := \mathcal{N}
    \left(\thetab|
        \mu=\thetab^*_1,
        \sigma^2=-(\nabla^2_\thetab\ln P(\thetab^*)^{-1}
    \right),
\end{equation}
which has density
\[
Q(\thetab) := \frac{P(\thetab^*)}{Z_Q}
e^{- \frac{1}{2}\nabla\ln P^\top \nabla^2\ln P^{-1} \nabla\ln P} 
e^{\frac{1}{2}(\thetab-\thetab^*_1)^\top (\nabla^2_\thetab\ln P (\thetab^*))^{-1}(\thetab-\thetab^*_1)},
\]
where $Z_Q=P(\thetab^*) e^{- \frac{1}{2}\nabla\ln P^\top \nabla^2\ln P^{-1} \nabla\ln P}  \sqrt{ (-2\pi)^{D} \det(\nabla^2_\thetab\ln P (\thetab^*))}$ is the normalizing constant. Then, $Q(\thetab)$ is a good approximation of $P(\thetab)$ in the sense that
\begin{align*}
    \ln Q(\thetab) + \ln Z_Q 
    & = \ln P(\thetab^*) - \frac{1}{2}\nabla\ln P^\top \nabla^2\ln P^{-1} \nabla\ln P + \frac{1}{2}(\thetab-\thetab^*_1)^\top \nabla^2\ln P(\thetab-\thetab^*_1 ) = \\
    & = \ln P(\thetab^*) + \nabla\ln P(\thetab-\thetab^*) + \frac{1}{2}(\thetab-\thetab^*)^\top \nabla^2\ln P(\thetab-\thetab^*) \cong \\
    & \cong \ln P(\thetab^*) + \nabla\ln P(\thetab-\thetab^*) + \frac{1}{2}(\thetab-\thetab^*)^\top \nabla^2\ln P(\thetab-\thetab^*) +  \mathcal{O}(\|\thetab-\thetab^*\|^3) = \\
    & =\ln P(\thetab).
\end{align*}
We highlight four points: 

(1) In order to ensure that the normalizing constant $Z_Q$ exists and consequently that $\textsc{laplace}(\thetab^*;P)$ is well defined, the hessian $\nabla^2_\thetab\ln P (\thetab^*)$ must be negative semi-definite. Differently from the locally maximal case, this is no longer guaranteed. 

(2) The method is numerically unstable if $\nabla^2_\thetab\ln P (\thetab^*)$ has eigenvalues close to 0. This has been empirically observed to commonly be the case \citep{sagun2016eigenvalues}. 

(3) We emphasize that the Taylor expansion is accurate around $\thetab^*$, but the Laplace Gaussian is centered in $\thetab^*_1$. This is often referred to as ``mean shift'' and implies that the sampled parameters from the normal distribution over weights are sampled far away from the actual mean. 

(4) Now, assume $\ln P(\thetab)$ to be the composition of a function $f(\thetab)$ and a loss $l(f)$. If, as in our setting, $f$ is linear and $l$ is concave, then $\ln P(\thetab)$ is concave, its hessian is guaranteed to be negative semi-definite, and the Laplace approximation is well-defined.

\section{Extended related work}\label{sec:extended_related_works}
Our proposed method's update rule resembles different methods in the literature. This is rather interesting since, despite arriving at a similar algorithm, they follow different derivations. In this section, we seek to explain the nuances between existing methods and ours. 

It is useful to recall our precision update rule
\begin{equation}
    \Hb_{t+1}
    =
    (1-\alpha) \Hb_t
    + 
    J_\thetab f^t_\cdot(\xc)^\top
    \Sigma_\text{data}^{-1}
    J_\thetab f^t_\cdot(\xc)
    \qquad
    \Hb_0 = \Sigma_\text{prior}^{-1}.
\end{equation}

To the best of our knowledge, a key conceptual difference with all related works in the literature is \textbf{Exact vs.\@ Approximate Hessian:} The second term on the RHS, $J^\top \Sigma^{-1} J$ is identical in all the update formulations but comes from a different derivation. In the formulation of \citet{daxberger2021laplace} (and others) this is an approximation of the hessian, that \emph{happens to guarantee negative definiteness}. In our formulation, this term comes from the linearization. This term is the exact hessian, i.e.\@ no approximation, such that negative definiteness is implied by the linearization. This results in a more intuitive understanding of the linearization error.  

\subsection{Differences with Laplace Redux}

\citet{daxberger2021laplace} specifically develop a post-hoc method, which assumes access to a MAP parameter, nevertheless, they also present an online training scheme. This appears to be grounded in ideas from second-order optimization, while ours is closer linked to the probabilistic model.

Neglecting their prior optimization procedure, \citet{daxberger2021laplace} considers
\begin{equation}
    \Hb_{t+1}
    =
    \Sigma_\text{prior}^{-1}
    + 
    J_\thetab f^t_\cdot(\xc)^\top
    \Sigma_\text{data}^{-1}
    J_\thetab f^t_\cdot(\xc).
\end{equation}
We highlight two main differences.

\textbf{Iteratively Updating the Hessian vs.\@ Using an Uninformed Prior:} In our formulation the first term of the RHS is the previous precision, which we ``discount'' with a forgetting term $\alpha$ that comes from the error inherited by moving away from the previous linearization. In the formulation of \citet{daxberger2021laplace} this is a prior that is assumed to be in the form $\gamma^2\mathbb{I}$, where the scalar $\gamma$ is optimized at every step through maximization of the evidence (Eq.~6 in their paper). This evidence maximization comes with some drawbacks: (1) It tends to make the Laplace approximation overconfident to outliers. They partially address this issue by adding to the evidence an auxiliary term that depends on an \textsc{ood} dataset, penalizing it (Eq.~12 in their Appendix). (2) Besides inducing the avoidable need for an \textsc{ood} dataset, this technique has the same pitfall as \textsc{vae}s, namely uncertainty should be a derived quantity, not a learned one. 

\textbf{Computing the hessian at every iteration vs.\@ every epoch:} We update the hessian estimate every iteration. In practice, \citet{daxberger2021laplace}  updates the hessian every epoch. In principle, they could update the hessian more regularly, but to the best of our knowledge, they do not explore this path.

% ...






\subsection{Differences with Variational Adaptive Newton Method}
\citet{Khan2017Variationaladaptivenewton} propose the following precision update rule
\begin{equation}
    \Hb_{t+1}
    =
    \Hb_t
    + 
    \alpha J_\thetab f^t_\cdot(\xc)^\top
    \Sigma_\text{data}^{-1}
    J_\thetab f^t_\cdot(\xc).
\end{equation}
The updates are very similar to ours, where the only difference is that the scaling is made on the Jacobian product instead of on the previous precision $\Hb_t$. Both their and our updates are ``additive'' in the sense that the magnitude is increasing, for them strictly monotonically, for us on average depending on the magnitude of $\alpha$. Thus, in both cases, the variance will approach 0 in the limit, modeling epistemic uncertainty disappearing for infinitely long training.

Their update rule comes from the Variational Optimization setting, which can be viewed as an instance of Variational Inference neglecting the KL term. They highlight strong similarities with Newton's method.

\subsection{Differences with Noisy Natural Gradient}
\citet{Zhang2017NoisyNaturalGradient} propose the precision update rule
\begin{equation}
    \Hb_{t+1}
    =
    (1-\alpha)\Hb_t
    + 
    \alpha J_\thetab f^t_\cdot(\xc)^\top
    \Sigma_\text{data}^{-1}
    J_\thetab f^t_\cdot(\xc).
\end{equation}
Again, the update rule is very similar to ours, besides the scaling being applied to the Jacobian product. They use a convex sum, which makes the update rule ``norm preserving''. Thus, these updates have very different asymptotic behavior than ours. 

Their update rule comes from applying, in the context of Variational Inference, natural gradient to the variational posterior distribution, instead of directly on the parameters space. The natural gradient is deeply connected with \textsc{lae} and natural parameters highlight the importance of updating the precision matrix instead of the covariance.

They also extend their derivation to the KFAC approximation of the hessian (in place of the exact diagonal), this is made through the use of matrix variate Gaussian. Despite not being considered in this work, a similar derivation is in principle feasible in our setting.

\subsection{Connection with Adam}

Both \citet{Zhang2017NoisyNaturalGradient} and \citet{Khan2017Variationaladaptivenewton} highlight strong similarities with the Adam method~\citep{KingmaB14adam}. We share these similarities and we highlight them too. The ``connection point'' is a noisy version of Adam. \citet{Zhang2017NoisyNaturalGradient} describe this method and call it ``Noisy Adam'' (Algorithm 1 in their paper). The difference from the vanilla version is that at each step, instead of using the current parameter, they use a noisy version of it. The noise magnitude is the pseudo-second-order term ($v_t$ in the original Adam paper~\citep{kingma2015adam}). 

We can then interpret Noisy Adam as an instance of our variational setting, specifically where the expectation estimate $\mathbb{E}_{q(\theta)}[\cdot]$ is made through Monte Carlo estimation with $N=1$ samples. Having the methods in the same setting, we can compare them and highlight the two main differences.

First, we emphasize the difference is in the second order term. Similarly to \citet{Zhang2017NoisyNaturalGradient} and \citet{Khan2017Variationaladaptivenewton}, we use the diagonal of the hessian (technically the diagonal of the GGN), while Adam uses the pointwise square of the gradient, which they call the second raw moment and is intended as a cheaply computable approximation of the Hessian.

Another difference is that Adam applies the square root. This is a minor point since, as pointed out by \citet{Zhang2017NoisyNaturalGradient}, this change may affect optimization performance, but does not change the fixed points.

\subsection{Connection with Bayes by Backpropagation}
Bayes by Backprop~\citep{blundell2015bayesbybackprop} can be viewed as a sample-based approximation of Laplace. In order to show this relation, we recall two very powerful equations \citep{Opper09VariationalGaussianApproximationRevisited}. Specifically, let $\mu$, $\Sigma$ be the parameter of a Gaussian distribution $q(\theta)\sim\mathcal{N}(\mu,\Sigma)$, and let $V(\theta)$ be an arbitrary $L^2$ integrable function (the log-likelihood in our case). Then, we are interested in the derivatives of $\mathbb{E}_{\theta\sim q}[V(\theta)]$. By standard Fourier analysis and integration by part, we have
\begin{equation}
    \nabla_\mu \mathbb{E}_{\theta\sim q}[V(\theta)]
    =
    \mathbb{E}_{\theta\sim q}[\nabla_\theta V(\theta)],
\end{equation}
\begin{equation}
    \nabla_\Sigma \mathbb{E}_{\theta\sim q}[V(\theta)]
    =
    \frac{1}{2}\mathbb{E}_{\theta\sim q}[\nabla^2_\theta V(\theta)].
\end{equation}
While the first equation is somehow trivial, the second highlight a very deep relationship. The LHS can be rewritten as
\begin{equation}
     \nabla_\Sigma \mathbb{E}_{\theta\sim q}[V(\theta)]
     =
     \nabla_\Sigma \mathbb{E}_{\epsilon\sim \mathcal{N}(0,1)}[V(\mu + \epsilon\Sigma)]
     =
     \mathbb{E}_{\epsilon\sim \mathcal{N}(0,1)}[ \nabla_\Sigma V(\mu + \epsilon\Sigma)].
\end{equation}
We can recognize that inside the expectation on the RHS is exactly the update rule for the variance in the bayes by backprop method~\citep{blundell2015bayesbybackprop}. Thus, we can interpret bayes by backprop as a one-sample Monte Carlo estimation of the expected value. The equations shows that in the limit of infinite samples, the bayes by backprop update step $\nabla_\Sigma V(\mu + \epsilon\Sigma)$ converges to the (averaged) Laplace approximation step $\mathbb{E}_{\theta\sim q}[\nabla^2_\theta V(\theta)]$ up to a factor 2. This motivates both the power and the instability of bayes by backprop.

\section{Model}\label{sec:model_appendix}

\subsection{Overview}
Let $\mathbb{X}=\mathbb{R}^D$ be the \emph{data} space, let $\mathbb{Y}=\mathbb{R}^D$ be the \emph{reconstruction} space and let $\Theta=\mathbb{R}^P$ be the \emph{parameter} space.
Let $\mathbb{F} = \left(\Theta \rightarrow \left(\mathbb{X}\rightarrow\mathbb{Y}\right)\right)$ be the space of operators from $\Theta$ to the space of operators from $\mathbb{X}$ to $\mathbb{Y}$. We will denote a function $f\in\mathbb{F}$ applied to a parameter $\thetab\in\Theta$ as $f_\thetab:\mathbb{X}\mapsto\mathbb{Y}$. This will represent, for example, a \textsc{nn} $f$ with a specific set of parameter (weights) $\thetab\in\Theta$, that maps some data $\xc\in\mathbb{X}$ to some reconstruction $f_\thetab(\xc) = \yc \in \mathbb{Y}$.

Let $\mathbb{X}\times\mathbb{Y}\times\Theta\times\mathbb{F}$ be a probability space. The only assumption we make on this space is
\begin{equation}\label{eq:reconstruction_hp}
    p(\yc|\xc,\thetab,f) \sim \mathcal{N}(\yc|\mu=f_\thetab(\xc), \sigma^2=\Sigma)
    \qquad
    \forall (\xc,\thetab,f)\in\mathbb{X}\times\Theta\times\mathbb{F}
\end{equation}
where $\Sigma\in\mathfrak{M}(\mathbb{R}^D\times\mathbb{R}^D)$ is a \emph{fixed} variance matrix. This is a common assumption for regression tasks, and is sometimes referred to as the ``data noise'' or ``reconstruction error''. In this paper, we fix $\Sigma=\mathbb{I}$, but the derivations hold for the general case. With only this assumption, the distribution is undefined and multiple solutions can exist. Thus, we require more assumptions.

A dataset $\mathcal{D}=\{\xc_n\}$ is a finite of infinite collection of data $\xc_n\in\mathbb{X}$ that is assumed to follow a certain, fixed but unknown, distribution
\begin{equation}
    \xc_n \sim p(\xc).
\end{equation}

Sacrificing slim notation for the sake of clarity, we introduce an operator $\mathcal{I}:\mathbb{X}\mapsto\mathbb{Y}$. This represents the ideal reconstruction for a given $\xc$. In the standard supervised setting, this would be the operator (defined on the dataset only) that maps each data input to its label. In our unsupervised setting, where $\mathbb{X}$ and $\mathbb{Y}$ are \emph{the same space}, the operator $\mathcal{I}$ is simply the identity (Indeed they are \emph{not} the same space, they are isomorphic spaces that we identify through the operator $\mathcal{I}$). Since $\mathcal{I}$ is the identity, it is often neglected in the literature, which can lead to unclear and potentially ambiguous Bayesian derivations. Thus, we choose to adopt this heavier, but more precise notation.

We assume access to a specific $f^{NN}\in\mathbb{F}$. Practically this will be our \textsc{nn} architecture, i.e. an operator that, given a set of parameters $\thetab\in\Theta$ gives rise to a function from $\mathbb{X}$ to $\mathbb{Y}$. Having $f^{NN}$ fixed, one may consider $f$ not to be stochastic anymore, we choose to still explicitly condition on $f$ in order to have a clearer notation in later stages. Note that, despite not being covered in this work, a proper stochastic derivation also on the \textsc{nn} architecture should be feasible. 

\subsection{Objective}
The \textsc{nn}'s parameter optimization process in this full Bayesian probabilistic framework can be viewed as: given a fixed $f^{NN}\in\mathbb{F}$, namely the \textsc{nn} architecture, maximise the reconstruction probability of $\mathcal{I}(\xc_n)$ over the dataset $\mathcal{D}$
\begin{equation}
    \mathbb{E}_{\xc_n\sim p(\xc)}
    \left[p(\yc|\xc_n,f^{NN})\big|_{\yc=\mathcal{I}(\xc_n)}
    \right]
    =
    \sum_{\xc_n\in\mathcal{D}}
    p(\yc|\xc_n,f^{NN})\big|_{\yc=\mathcal{I}(\xc_n)},
\end{equation}
where the untractable $p(\yc|\xc_n,f^{NN})$ can be expanded in $\thetab$ and thus related to our hypothesis (\ref{eq:reconstruction_hp}) as
\begin{equation}
    p(\yc|\xc_n,f^{NN})
    =
    \mathbb{E}_{\thetab\sim p(\thetab|\xc_n,f^{NN})}
    \left[ p(\yc|\xc_n,\thetab,f^{NN})
    \right].
\end{equation}
Notice that the only unfixed quantity is the distribution on the parameters, which we will optimize for. We are not interested in finding a datapoint-dependant distribution, but rather one that maximise all reconstructions at the same time, i.e. $p(\thetab|f^{NN}) = p(\thetab|\xc_n,f^{NN})$. We can then frame Bayesian optimization as: find a distribution on parameters such that
\begin{align}
    q(\thetab)
    & \in
    \arg\max_{p(\thetab|f^{NN})}
    \sum_{\xc_n\in\mathcal{D}}
        \mathbb{E}_{\thetab\sim p(\thetab|f^{NN})}
        \left[ p(\yc|\xc_n,\thetab,f^{NN})\big|_{\yc=\mathcal{I}(\xc_n)}
        \right] \\
    & = \arg\max_{p(\thetab|f^{NN})}
    \sum_{\xc_n\in\mathcal{D}}
        \mathbb{E}_{\thetab\sim p(\thetab|f^{NN})}
        \left[ p(\mathcal{I}(\xc_n)\big|\mathcal{N}(f^{NN}_\thetab(\xc_n),\Sigma))
        \right].
\end{align}
Moreover, finding this optimum in the space $\Delta(\Theta)$ of all distributions on $\Theta$ is not tractable. So, as commonly done, we restrict ourselves to the subset $\mathcal{G}(\Theta)\subset\Delta(\Theta)$ of Gaussians over $\Theta$. Then, a \emph{solution} in our context is
\begin{equation}\label{eq:bayesian_setting_objective}
    q(\thetab)
    \in
    \arg\max_{q\in\mathcal{G}(\Theta)}
    \sum_{\xc_n\in\mathcal{D}}
        \mathbb{E}_{\thetab\sim q(\thetab)}
        \left[ p(\mathcal{I}(\xc_n)\big|\mathcal{N}(f^{NN}_\thetab(\xc_n),\Sigma))
        \right].
\end{equation}
We emphasize that this solution has no guarantees of being unique, but we are interested in finding one of them. 

\subsection{Joint distribution for a fixed datapoint}
Let us first get a better understanding of the joint distribution on $\mathbb{Y}\times\Theta$ conditional to a fixed datapoint $\xc\in\mathbb{X}$ and a network architecture $f\in\mathbb{F}$
\begin{equation}
    p(\yc,\thetab|\xc,f).
\end{equation}
This distribution has two \emph{marginals}
\begin{equation}\label{marginal_y}
    p(\yc|\xc,f)
\end{equation}
\begin{equation}\label{marginal_theta}
    p(\thetab|\xc,f)
\end{equation}
and two \emph{conditionals}
\begin{equation}\label{conditional_y}
    p(\yc|\thetab,\xc,f)
\end{equation}
\begin{equation}\label{conditional_theta}
    p(\thetab|\yc,\xc,f).
\end{equation}
These four quantities must satisfy the system of two ``recursive'' equations
\begin{align}\label{eq:marginal_y}
    p(\yc|\xc,f)
    & = \int_\Theta p(\yc|\thetab,\xc,f) p(\thetab|\xc,f) \text{d}\thetab = \\
    & = \mathbb{E}_{\thetab\sim p(\thetab|\xc,f)} [p(\yc|\thetab,\xc,f)] \nonumber
\end{align}
\begin{align}\label{eq:marginal_theta}
    p(\thetab|\xc,f)
    & = \int_\mathbb{Y} p(\thetab|\yc,\xc,f) p(\yc|\xc,f) \text{d}\yc = \\
    & = \mathbb{E}_{\yc\sim p(\yc|\xc,f)} [p(\thetab|\yc,\xc,f)]. \nonumber
\end{align}
If these are satisfied then the joint is a well-defined distribution and we can apply Bayes rule
\begin{equation}
    p(\yc|\thetab,\xc,f) p(\thetab|\xc,f)
    =
    p(\yc,\thetab|\xc,f)
    =
    p(\thetab|\yc,\xc,f) p(\yc|\xc,f)
\end{equation}
which in logarithmic form is
\begin{equation}\label{eq:bayes_rule_log}
    \log p(\yc|\thetab,\xc,f) + \log p(\thetab|\xc,f)
    =
    \log p(\thetab|\yc,\xc,f) + \log p(\yc|\xc,f).
\end{equation}
We can factor in the assumptions. The ``data noise'' \emph{assumption} gives us one of the two conditionals:
\begin{equation}\label{eq:assumption_datanoise}
    p(\yc|\thetab,\xc,f) 
    \sim \mathcal{N}(\yc|\mu=f_\thetab(\xc), \sigma^2=\Sigma).
\end{equation}
The ``Gaussian parameter'' \emph{assumption} gives us one of the two marginals: 
\begin{equation}\label{eq:assumtion_gaussianparameters}
    p(\thetab|\xc,f) 
    = q^t(\thetab)
    \sim \mathcal{N}(\thetab|\mu=\thetab_t, \sigma^2=\Hb_t^{-1}).
\end{equation}
With these in place, the joint distribution is uniquely defined. The other marginal, by Eq.~\ref{eq:marginal_y}, is 
\begin{equation}\label{eq:marginal_y_untractable}
    p(\yc|\xc,f)
    = \mathbb{E}_{\thetab\sim q^t(\thetab)} [p(\yc|\mathcal{N}(f_\thetab(\xc), \Sigma))]
\end{equation}
and the other conditional, by Bayes rule, is
\begin{equation}
    p(\thetab|\yc,\xc,f) = \frac{p(\yc|\thetab,\xc,f) p(\thetab|\xc,f)}{p(\yc|\xc,f)}.
\end{equation}

Despite being uniquely defined, the integral in Eq.~\ref{eq:marginal_y_untractable} is, with a general $f$, intractable, and so is the joint distribution.

\textbf{But why do we even care?} The intractability of $p(\yc|\xc,f)$ in Eq.~\ref{eq:marginal_y_untractable} may at first glance appear irrelevant. This is the case, for instance, with bayes by backprop \citep{blundell2015bayesbybackprop} methods. They simply need access to the gradient of this quantity. For this purpose, a simple Monte Carlo estimate of the expectation is enough.

On the other hand, we are interested in recovering a meaningful distribution on parameters. This imply that we aim at using Eq.~\ref{eq:marginal_y_untractable} to enforce that Eq.~\ref{eq:marginal_theta} holds. For this purpose we need access to the the density $p(\yc|\xc,f)$, so a Monte Carlo estimate of Eq.~\ref{eq:marginal_y_untractable} is not enough.

\subsection{Linear $f$}\label{sec:proof_covariance}
\begin{theorem}
Given the data noise assumption from Eq.~\ref{eq:assumption_datanoise}
\begin{equation}
p(\yc|\thetab,\xc,f)\sim \mathcal{N}(\yc|\mu=f_\thetab(\xc), \sigma^2=\Sigma),
\end{equation}
given the Gaussian parameter assumption from Eq.~\ref{eq:assumtion_gaussianparameters} for some $\thetab_t$, $\Hb_t$
\begin{equation}p(\thetab|\xc,f) = q^t(\thetab) \sim \mathcal{N}(\thetab|\mu=\thetab_t, \sigma^2=\Hb_t^{-1}),
\end{equation}
assume that $f$ is linear in $\thetab$, i.e.
\begin{equation}
    f_\thetab(\xc) = f_0(\xc) + J f_\cdot(\xc) \thetab
    \qquad\forall\thetab\in\Theta,\forall \xc\in\mathbb{X}.
\end{equation}
Then the joint distribution is Gaussian itself
\begin{equation}
    p(\thetab,\yc|\xc,f) \sim \mathcal{N}((\thetab,\yc)|\mu_t,\Sigma_t)
\end{equation}
where
\begin{equation*}
    \mu_t
    = \begin{pmatrix}
           \thetab_t \\
           f_{\thetab_t} (\xc)
         \end{pmatrix},
   \text{ and }
   \Sigma_t = \begin{pmatrix}
     \Hb_t^{-1} 
     & \Sigma J f_\cdot(\xc)^{\top} \\
     J f_\cdot(\xc) \Sigma
     & \left( J f_\cdot(\xc)^{\top} \Hb_t J f_\cdot(\xc) \right)^{-1} + \Sigma
   \end{pmatrix}.
\end{equation*}
\end{theorem}
\begin{proof}
With the further assumption of linearity of $f$, we can explicitly carry out the integral in the expectation in Eq.~\ref{eq:marginal_y_untractable}
\begin{align*}
    p(\yc|\xc,f)
    & = \mathbb{E}_{\thetab\sim q^t(\thetab)} [p(\yc|\mathcal{N}(f_\thetab(\xc), \Sigma))] = \\
    & = \int_\Theta p(\yc|\mathcal{N}(f_\thetab(\xc), \Sigma)) q^t(\thetab) \text{d}\thetab= \\
    & = \int_\Theta p(\yc|\mathcal{N}(f_\thetab(\xc), \Sigma)) p(\thetab|\mathcal{N}(\thetab_t,\Hb_t^{-1})) \text{d}\thetab= \\
    & = \int_\Theta p(\yc|\mathcal{N}(f_0(\xc)+J f_\cdot(\xc) \thetab, \Sigma)) p(\thetab|\mathcal{N}(\thetab_t,\Hb_t^{-1})) \text{d}\thetab= \\
    & = p(\yc|\mathcal{N}(f_{\thetab_t}(\xc), (J f_\cdot(\xc)^\top \Hb_t J f_\cdot(\xc))^{-1} + \Sigma )).
\end{align*}
We emphasize that as a consequence $\nabla^2_\thetab \log p(\yc|\xc,f)$ is not dependent on $\thetab_t$.
\end{proof}


Having Theorem 1 in place, we can go back to our original problem. We need to deal with a fixed \emph{non-linear} architecture $f^{NN}$. We can exploit Theorem 1 by defining $f^t$: a linearization of $f^{NN}$ with a first-order Taylor expansion around $\thetab_t$
\begin{align}\label{eq:taylor_fNN}
    f^t_\thetab(\xc)
    & := \textsc{Taylor}(f^{NN},\thetab_t)(\xc) = \nonumber\\
    & = f^{NN}_{\thetab_t}(\xc) + J_\thetab f^{NN}_{\thetab_t}(\xc) (\thetab-\thetab_t)
\end{align}
and it holds that
\begin{equation}\label{eq:taylor_fNN_error}
    f^{NN}_\thetab(\xc) = f^t_\thetab(\xc) + \mathcal{O}(\|\thetab-\thetab_t\|^2).
\end{equation}
Recalling Eq.~\ref{eq:reconstruction_hp} both for $f^{NN}$ and for $f^t$
\begin{equation}
    p(\yc|\xc,\thetab,f^{NN}) \sim \mathcal{N}(\yc|\mu=f^{NN}_\thetab(\xc), \sigma^2=\Sigma)
\end{equation} 
\begin{equation}
    p(\yc|\xc,\thetab,f^t) 
    \sim \mathcal{N}(\yc|\mu=f^t_\thetab(\xc), \sigma^2=\Sigma)
\end{equation}
that, together with Eq.~\ref{eq:taylor_fNN_error}, imply
\begin{equation}\label{eq:random_equation_with_Oterm}
    \mathcal{N}\big(\yc|\mu=f^{NN}_\thetab(\xc), \sigma^2=\Sigma\big)
    \sim
    \mathcal{N}\big(\yc|\mu=f^t_\thetab(\xc) + \mathcal{O}(\|\thetab-\thetab_t\|^2), \sigma^2=\Sigma\big),
\end{equation}
where we can interpret the unknown $\mathcal{O}(\|\thetab-\thetab_t\|^2)$ as $\thetab$-dependent noise. More specifically, calling $\gamma>0$ the scalar constant of the $\mathcal{O}$-term, we assume that
\begin{equation}
    \mathcal{O}(\|\thetab-\thetab_t\|^2) 
    \approx
    \epsilon (\thetab-\thetab_t)^2
    \qquad
    \text{where }\epsilon\sim\mathcal{N}(0,\gamma\mathbb{I})
\end{equation}
and thus, from Eq.~\ref{eq:random_equation_with_Oterm}, we have
\begin{equation}
   \mathcal{N}\big(\yc|\mu=f^{NN}_\thetab(\xc), \sigma^2=\Sigma\big)
   \sim 
   \mathcal{N}(\yc|\mu=f^t_\thetab(\xc), \sigma^2=\Sigma + \gamma\|\thetab-\thetab_t\|^2\mathbb{I}).
\end{equation}
At this point, integrals are not analytically tractable, and thus a proper proof is not feasible, the intuition is that this increased variance reflects in increased variance in $p(\thetab|\xc,f^{NN})$

\begin{equation}\label{eq:bigO_error_in_hessian}
    \nabla^2_\thetab \log p(\thetab|\xc,f^{NN})\big|_{\thetab=\thetab_{t+1}} 
    \approx 
    \nabla^2_\thetab \log p(\thetab|\xc,f^t)\big|_{\thetab=\thetab_{t+1}} + \gamma\|\thetab_{t+1}-\thetab_t\|^2,
\end{equation}
where we introduce a hyperparameter $\alpha>0$ to cope with this added variance. If we then assume the Jacobian $J_\thetab f^{NN}$ to be a Lipschitz function, then the Lipschitz constant is an upper bound on $\gamma$, as follows from the Taylor expansion of \cref{eq:taylor_fNN}. If this Lipschitz constant is smaller than the inverse of the gradient step $1/\|\thetab_{t+1}-\thetab_t\|$ (which is not an unreasonable assumption for the gradient ascent to be stable) we have
\begin{equation}
    \gamma\|\thetab_{t+1}-\thetab_t\|^2 \approx \|\thetab_{t+1}-\thetab_t\|
\end{equation}
that gives us a plausible order of magnitude for choosing the hyperparameter $\alpha$.






\textbf{Motivation:} During training, we produce a sequence of Gaussians $q^t(\thetab)\sim\mathcal{N}(\thetab_t,\Hb_t^{-1})$ that we assume to be the distribution $p(\thetab|\xc,f^t)$, at every step $t\geq0$. This distribution $q^t$ is then used for (1) Gaussian derivation in the linear case, for (2) Monte Carlo sampling in the update rule $\thetab_t\rightarrow\thetab_{t+1}$ and (3), as second order derivative, for update rule $\Hb_t\rightarrow \Hb_{t+1}$. 

Moreover, given that this distribution $q^t$ is our "best guess so far", we assume it to be also the distribution $p(\thetab|\xc,f^{NN})$. This, being $f^{NN}$ not linear, (1) cannot be used for Gaussian derivation, (2) can reasonably be used for sampling (and thus we derive the improved update rule \cref{eq:improved_update_rule_theta}), and (3) can somehow be used as second order derivative (and thus we derive the improved update rule \cref{eq:improved_update_rule_H}), but the latter requires some more care. That is why we introduce the parameter $\alpha$.

\subsection{Iterative learning}
Our learning method produces a sequence of Gaussians \begin{equation}
    q^t(\thetab)\sim\mathcal{N}(\thetab|\mu=\thetab_t,\sigma^2=\Hb_t^{-1})
\end{equation} 
and a sequence of linearized functions
\begin{equation}
    f^t
    = \textsc{Taylor}(f^{NN},\thetab_t)
\end{equation}
for every $t\geq0$.

\textbf{Initialization} is trivially done using a Gaussian prior on the parameters.
\begin{equation}
    \thetab_0 = \thetab^\text{prior}
    \qquad
    \Hb_0 = (\Sigma^\text{prior})^{-1}
\end{equation}

\textbf{Iterative step} is made in two steps. First, having access to $\thetab_t$, we ``generate'' the linearization $f^t$. Practically this is equivalent to computing the two quantities $f^{NN}_{\thetab_t}(\xc)$ and $J_\thetab f^{NN}_{\thetab_t}(\xc)$, that, together with the value $\thetab_t$ are actually equivalent to ``generating'' $f^t$, as Eq.~\ref{eq:taylor_fNN} shows.

Second, we compute the Gaussian parameters $\thetab_{t+1}$ and $\Hb_{t+1}$.

Recalling our aim of maximizing the quantity \cref{eq:bayesian_setting_objective}, update on $q(\cdot)$ means, $\thetab_t\rightarrow\thetab_{t+1}$, is \emph{ideally} made through gradient ascent steps on $p(\yc|\xc,f)|_{\yc=\mathcal{I}(\xc)}$. As this is intractable, we instead do gradient steps on the lower bound $\mathcal{L}_{\yc}$ of (the $\log$ of) Eq.~\ref{eq:marginal_y}
\begin{equation}
    \thetab_{t+1} = \thetab_t + \lambda \nabla_\thetab \mathcal{L}_{\yc}\Big|_{\yc=\mathcal{I}(\xc)}
\end{equation}
where
\begin{equation}
    \mathcal{L}_{\yc} 
    = \mathbb{E}_{\thetab\sim p(\thetab|\xc,f^t)} [\log p(\yc|\thetab,\xc,f^t)] 
    \leq
    \log \mathbb{E}_{\thetab\sim p(\thetab|\xc,f^t)} [ p(\yc|\thetab,\xc,f^t)] 
\end{equation}
and so
\begin{align}\label{eq:theta_update_rule}
    \nabla_\thetab\mathcal{L}_{\yc} \big|_{\yc=\mathcal{I}(\xc)}
    & = \mathbb{E}_{\thetab\sim p(\thetab|\xc,f^t)} \left[\nabla_\thetab\log p(\yc|\thetab,\xc,f^t)\big|_{\yc=\mathcal{I}(\xc)}
    \right] = \nonumber\\
    & = \mathbb{E}_{\thetab\sim q^t(\thetab)}
    \left[ \nabla_\thetab \log p(\mathcal{I}(\xc)\big|\mathcal{N}(f^t_\thetab(\xc),\Sigma))
        \right].
\end{align}

Recalling the Laplace approximation \cref{eq:laplace_definition}, the negative precision, $-\Hb_{t+1}$, is \emph{ideally} set to be the the hessian of the log probability $p(\thetab|\xc,f)$, evaluated in $\thetab_{t+1}$. As this is intractable we instead set it to the hessian of the lower bound $\mathcal{L}_\thetab$ of (the log of) \cref{eq:marginal_theta}
\begin{equation}
    \Hb_{t+1} = - \nabla^2_\thetab \mathcal{L}_\thetab\Big|_{\thetab=\thetab_{t+1}}
\end{equation}
where
\begin{equation}
    \mathcal{L}_\thetab
    = \mathbb{E}_{\yc\sim p(\yc|\xc,f^t)} [\log p(\thetab|\yc,\xc,f^t)] 
    \leq
    \log \mathbb{E}_{\yc\sim p(\yc|\xc,f^t)} [ p(\thetab|\yc,\xc,f^t)] 
\end{equation}
and so
\begin{align*}
    \nabla^2_\thetab \mathcal{L}_\thetab\Big|_{\thetab=\thetab_{t+1}}
    & = \mathbb{E}_{\yc\sim p(\yc|\xc,f^t)} \left[\nabla^2_\thetab\log p(\thetab|\yc,\xc,f^t)\Big|_{\thetab=\thetab_{t+1}} 
    \right]  \\
    & \text{via \cref{eq:bayes_rule_log}} \\
    & = \mathbb{E}_{\yc\sim p(\yc|\xc,f^t)} 
    \Big[
    \nabla^2_\thetab\log p(\yc|\thetab,\xc,f^t)\Big|_{\thetab=\thetab_{t+1}} +
    \nabla^2_\thetab\log p(\thetab|\xc,f^t)\Big|_{\thetab=\thetab_{t+1}} + \\
    & \qquad\qquad -
    \nabla^2_\thetab\log p(\yc|\xc,f^t)\Big|_{\thetab=\thetab_{t+1}} 
    \Big] \\
    &\text{via Theorem 1} \\
    & = \mathbb{E}_{\yc\sim p(\yc|\xc,f^t)} 
    \left[
    \nabla^2_\thetab\log p(\yc|\thetab,\xc,f^t)\Big|_{\thetab=\thetab_{t+1}} +
    \nabla^2_\thetab\log p(\thetab|\xc,f^t)\Big|_{\thetab=\thetab_{t+1}}
    \right] \\
    & \text{via the chain rule \cref{eq:hessian_chain_rule}}\\
    & = \mathbb{E}_{\yc\sim p(\yc|\xc,f^t)} 
    \left[
    J_\thetab f^t_\cdot(\xc)^\top
    \nabla^2_{\yc}\log p(\yc|\thetab_{t+1},\xc,f^t)
    J_\thetab f^t_\cdot(\xc)
    +
    \nabla^2_\thetab\log p(\thetab|\xc,f^t)\Big|_{\thetab=\thetab_{t+1}}
    \right] \\
    & \text{via HP \cref{eq:reconstruction_hp}}\\
    & = \mathbb{E}_{\yc\sim p(\yc|\xc,f^t)} 
    \left[
    - J_\thetab f^t_\cdot(\xc)^\top
    \Sigma^{-1}
    J_\thetab f^t_\cdot(\xc)
    +
    \nabla^2_\thetab\log p(\thetab|\xc,f^t)\Big|_{\thetab=\thetab_{t+1}}
    \right] \\
    & = 
    - J_\thetab f^t_\cdot(\xc)^\top
    \Sigma^{-1}
    J_\thetab f^t_\cdot(\xc)
    +
    \nabla^2_\thetab\log p(\thetab|\xc,f^t)\Big|_{\thetab=\thetab_{t+1}} \\
    & = 
    - J_\thetab f^t_\cdot(\xc)^\top
    \Sigma^{-1}
    J_\thetab f^t_\cdot(\xc)
    +
    \nabla^2_\thetab\log q^t(\thetab)\Big|_{\thetab=\thetab_{t+1}} \\
    & = 
    - J_\thetab f^t_\cdot(\xc)^\top
    \Sigma^{-1}
    J_\thetab f^t_\cdot(\xc)
    - \Hb_t.
\end{align*}

\subsubsection{Improved update rule}
As said, the update $\thetab_t\rightarrow\thetab_{t+1}$ is \emph{ideally} made through gradient ascent steps on $p(\yc|\xc,f)|_{\yc=\mathcal{I}(\xc)}$, but we instead use the tractable lower bound with $f^t$. We can perform the same derivation using $f^{NN}$ in place of $f^t$. Assuming $p(\thetab|\xc,f^{NN})\sim q^t(\thetab)$ for sampling, leads to the improved update rule
\begin{equation}\label{eq:improved_update_rule_theta}
    \thetab_{t+1}
    = \thetab_t + 
    \lambda \mathbb{E}_{\thetab\sim q^t(\thetab)}
    \left[ \nabla_\thetab \log p(\mathcal{I}(\xc)\big|\mathcal{N}(f^{NN}_\thetab(\xc),\Sigma))
    \right].
\end{equation}

Similarly, the negative precision $-\Hb_{t+1}$ is \emph{ideally} set to be the hessian of the log probability $p(\thetab|\xc,f)$. but instead, we use the tractable lower bound with $f^t$. Here we cannot perform the same derivation, since Theorem 1 does not hold anymore. Instead, we can rely on the estimate \cref{eq:bigO_error_in_hessian} to improve the term 
\[
\nabla^2_\thetab\log p(\thetab|\xc,f^t)\Big|_{\thetab=\thetab_{t+1}} = -\Hb_t 
\qquad\longrightarrow\qquad 
\nabla^2_\thetab\log p(\thetab|\xc,f^{NN})\Big|_{\thetab=\thetab_{t+1}} \approx -(1-\alpha)\Hb_t,
\]
and this leads to the improved update rule
\begin{equation}\label{eq:improved_update_rule_H}
    \Hb_{t+1}
    =
    (1-\alpha) \Hb_t
    + J_\thetab f^t_\cdot(\xc)^\top
    \Sigma^{-1}
    J_\thetab f^t_\cdot(\xc).
\end{equation}

\section{Fast Hessian}\label{sec:fast_hessian_appendix}
We are interested in computing the hessian of a loss function. For this purpose, the Jacobian of the \textsc{nn} w.r.t. parameters plays a crucial role. In this section we develop a better understanding of this object, we derive the backpropagation (also used by the BackPack library~\citep{dangel2020backpack}) and finally, we explain our approximated backpropagation that allows linear scaling.

\subsection{Jacobian of a Neural Network}
Let us first define some terminology that we will need for chain rule derivations. A \textsc{nn} is a composition of $l$ functions $f:=f_{L_l}\circ f_{L_{l-1}} \circ\,\dots\,\circ f_{L_2} \circ f_{L_1}$:
\[
x_0 \underset{f_{L_1}}{\xrightarrow{\hspace*{1cm}}} x_1 \longrightarrow 
\quad\dots\quad 
\longrightarrow 
x_{i-1} \underset{f_{L_i}}{\xrightarrow{\hspace*{1cm}}} x_i \longrightarrow 
\quad\dots\quad 
\longrightarrow
x_{l-1} \underset{f_{L_l}}{\xrightarrow{\hspace*{1cm}}} x_l
\]
where there are parametric and non-parametric function $f_{L_i}$. We here highlight two common parametric functions and a common non-parametric function.

\textbf{Parametric function} such as a linear layer 
\[ x_i = 
    f_{L_i}(x_{i-1}) = 
    \phi_{L_i} x_{i-1} \quad
    \text{ where }\phi_{L_i}\in\mathfrak{M}(|x_i|,|x_{i-1}|)
\]
or convolution
\[ x_i = 
    f_{L_i}(x_{i-1}) = \texttt{conv}_{\text{feat}=\phi_{L_i}} (x_{i-1}) \quad
    \text{ where }\phi_{L_i}\in\mathfrak{M}(\text{in channel}, \text{out channel}, \text{feat height}, \text{feat width})
\]
\textbf{Non-parametric} such as activation functions $L_i=\texttt{tanh},\texttt{ReLU}\dots$
\[ x_i = 
    f_{L_i}(x_{i-1}) \quad
    \text{ where }|x_i|=|x_{i-1}| 
\]

What is conventionally called \textit{layer} is actually a composition of two functions: a linear function and an activation function. For sake of clarity in our derivation, we do not adopt this convention and we use the word ``layer'' to indicate the singular ``function'' component of the \textsc{nn}.

Let us now consider a \textsc{nn} with $w$ parametric layers and $l-w$ activation layers
\[ x_l = f_{\phi}(x_0) = f_{L_l}\circ\dots\circ f_{L_1}(x_0)
\qquad \text{where }\phi=(\phi_1,\dots,\phi_w),\]
and define the bijection
\[\begin{array}{lcccl}
    \mathcal{W}:
        & \{1,\dots,w\}
        & \longrightarrow 
        & \{i | \text{ s.t. } L_i \text{ is parametric layer} \}
        & \subseteq \{1,\dots,l\} \\
        & p
        & \longmapsto 
        & i \text{ s.t. } L_i \text{ has parameters } \phi_p
        &
\end{array}\]
from the subset of parametric layers to the corresponding index in $\phi=(\phi_1,\dots,\phi_w)$.

The Jacobian w.r.t. the parameters $J_\phi f_\phi(x_0)\in\mathfrak{M}(|x_l|,|\phi_1|+\dots+|\phi_w|)$
is a matrix with number-of-output $|x_l|$ rows and number-of-parameters $|\phi|$ columns. For reference, just storing this matrix can exceed memory limits even with the smallest autoencoder working on \textsc{mnist}.

There are two ways of looking at this matrix: (1) \textbf{row by row}, that is output by output or (2) \textbf{block-of-columns by block-of-columns}, that is layer by layer. 

\subsection{Jacobian per output}
Each row of the Jacobian corresponds to the gradient w.r.t. the parameters of an element of the output
\[
J_\phi f_\phi(x_0) = 
\left(\begin{array}{c}
    \nabla_\phi [f_\phi(x_0)]_1 \\
    \hline
    \vdots \\
    \hline
    \nabla_\phi [f_\phi(x_0)]_{|x_l|} 
\end{array}\right)
\]
This can be computed by defining $|x_l|$ loss functions
\[ loss_k(x_l) := [x_l]_k \quad\text{for }k=1,\dots,|x_l| \]
and backpropagating each of those to obtain one line at a time. The disadvantage of this formulation is that we cannot reuse computation for one loss function to improve the computation of other loss functions (they are \textit{independent}). Moreover, we need to store all these rows at the same time in order to compute $J^\top J$, which is computationally impractical.

\subsection{Jacobian per layer}
Each column of the Jacobian is the derivative of the output vector w.r.t.\@ a single parameter. We can then group the parameters (i.e. columns) layer by layer
\[
J_\phi f_\phi(x_0) = 
\left(\begin{array}{c|c|c}
    & & \\
    J_{\phi_1}f_\phi(x_0) &
    \,\dots\, &
    J_{\phi_w}f_\phi(x_0) \\
    & &
\end{array}\right)
\]
where $J_{\phi_p}f_\phi(x_0)\in\mathfrak{M}(|x_l|,|\phi_p|)$. Let us focus on the computation of $J_{\phi_p}f_\phi(x_0)$ for a fixed layer $p=1,\dots,w$. First notice that the parameters $\phi_p$ in $f_\phi=f_{L_l}\circ \,\dots\,\circ f_{L_1}$ only appear in $f_{L_{\mathcal{W}(p)}}$ and so
\[ \frac{\partial f_{L_i}}{\partial \phi_p} (x_{i-1}) = 0 \text{ if }i\not=\mathcal{W}(p). \]\\
\textbf{Chain rule} (informal)
\begin{align*}
    \frac{\partial f_\phi(x_0)}{\partial\phi_p}
        &= \frac{\partial x_l}{\partial\phi_p} = \\
        &= \frac{\partial x_l}{\partial x_{l-1}} \,
            \frac{\partial x_{l-1}}{\partial\phi_p}= \\
        &=  \underbrace{\frac{\partial x_l}{\partial x_{l-1}} }_{\substack{\text{layer $l$} \\ \text{w.r.t. input}}}
            \underbrace{\frac{\partial x_{l-1}}{\partial x_{l-2}} }_{\substack{\text{layer $l-1$} \\ \text{w.r.t. input}}}
            \quad\dots\quad
            \underbrace{\frac{\partial x_{\mathcal{W}(p)+1}}{\partial x_{\mathcal{W}(p)}} }_{\substack{\text{layer $\mathcal{W}(p)+1$} \\ \text{w.r.t. input}}}
            \underbrace{\frac{\partial x_{\mathcal{W}(p)}}{\partial \phi_p} }_{\substack{\text{layer $\mathcal{W}(p)$} \\ \text{w.r.t. parameters}}}
\end{align*}\\
\textbf{Chain rule} (formal)
\begin{equation}\label{jacobian_chain_rule}
J_{\phi_p} f_\phi(x_0) =
\left(\prod_{\mathcal{W}(p)+1}^{k=l} 
    J_{x_{k-1}}f_{L_k}(x_{k-1}) \right)
J_{\phi_p} f_{L_{\mathcal{W}(p)}}(x_{\mathcal{W}(p)-1})
\end{equation}
The intuition for the chain rule is that the Jacobian $J_{\phi_p} f_\phi(x_0)$ is the composition of the Jacobians w.r.t.\@ \emph{input} of subsequent layers times the Jacobian w.r.t.\@ \emph{parameters} of the specific layer. Thus, we can reuse computation for one layer to improve the computation of other layers, specifically the product of Jacobians w.r.t.\@ input.
Moreover, we can compute $J_p^\top J_p$ layer by layer without ever storing the full Jacobian.

\subsubsection{Jacobian of a layer w.r.t.\@ to input}
The Jacobian of a standard \textbf{linear layer} w.r.t.\@ to the input is 
\[ J_{x_{p-1}}f_{\phi_p}(x_{p-1}) = \phi_p \]
and this remains the same also in the case with a bias. The Jacobian of a \textbf{convolutional layer} w.r.t.\@ to the input is
\[
J_{x_{p-1}} \texttt{conv}_{\text{feat}=\phi_p} (x_{p-1})
= \mathcal{M}(\texttt{conv}_{\text{feat}=\phi_p})
\]
The Jacobian of the activation function depends on the specific choice. Recall that, for each layer $i$, $x_i\in\mathbb{R}^{|x_i|}$ is a vector
\[ x_i = \big( [x_i]_k \big)_{k=1,\dots,|x_i|} \]
where $[x_i]_k\in\mathbb{R}$ is the value in position $k$ of the vector $x_i$.

If $L_i$ is \textbf{tanh}
\[
[x_i]_k = 
[f_{L_i}(x_{i-1})]_k = 
\texttt{tanh}([x_{i-1}]_k) =
\frac{e^{[x_{i-1}]_k} - e^{-[x_{i-1}]_k}} {e^{[x_{i-1}]_k} + e^{-[x_{i-1}]_k}}
\qquad \text{for } k=1,\dots,|x_{i-1}|
\]
then
\[
[J_{x_{i-1}}f_{L_i}(x_{i-1})]_{kj} =
\delta_{kj} \left( 1 - (\texttt{tanh}([x_{i-1}]_k))^2 \right) =
\delta_{kj} \left( 1 - [x_i]_k^2 \right)
\qquad \text{for } k,j=1,\dots,|x_{i-1}|.
\]

If $L_i$ is \textbf{ReLU}
\[
[x_i]_k = 
[f_{L_i}(x_{i-1})]_k = 
\texttt{ReLU}([x_{i-1}]_k) =
\max(0, [x_{i-1}]_k )
\qquad \text{for } k=1,\dots,|x_{i-1}|
\]
then
\[
[J_{x_{i-1}}f_{L_i}(x_{i-1})]_{kj} =
\delta_{kj} \left( 1 \text{ if } [x_{i-1}]_k>0 \text{ else } 0 \right)
\qquad \text{for } k,j=1,\dots,|x_{i-1}|.
\]

\subsubsection{Jacobian of a layer w.r.t.\@ to parameters}
The Jacobian of a standard linear layer w.r.t.\@ to the parameters is
\[ J_{\phi_i}f_{\phi_i}(x_{i-1}) 
= \mathbb{I}_{|x_i|} \otimes x_{i-1} 
\quad\in\mathfrak{M}(|x_i|,|\phi_i|)\]
and in the case with bias $b_i\in\mathbb{R}^{x_i}$ the Jacobian is
\[ J_{\phi_i,b_i}f_{\phi_i,b_i}(x_{i-1}) 
= \mathbb{I}_{|x_i|} \otimes [x_{i-1},1] 
\quad\in\mathfrak{M}(|x_i|,|\phi_i|+|b_i|)\]
The Jacobian of a convolutional layer w.r.t. to the parameters is
\[
J_{\phi_i} \texttt{conv}_{\text{feat}=\phi_i} (x_{i-1})
= J_{\phi_i} \texttt{conv}^T_{\text{feat}=rev(x_{i-1})} (\phi_{i})
= \mathcal{M}(\texttt{conv}_{\text{feat}=rev(x_{i-1})})^T
\quad\in\mathfrak{M}(|x_i|,|\phi_i|)
\]
and in the case with bias $b_i\in\mathbb{R}^{o_i}$ the Jacobian is
\[ J_{\phi_i} \texttt{conv}_{\text{feat}=\phi_i,\text{bias}=b_i} (x_{i-1})
= \big( \mathcal{M}(\texttt{conv}_{\text{feat}=rev(x_{i-1})})^T \big| I \big)
\quad\in\mathfrak{M}(|x_i|,|\phi_i|+|b_i|)
\]

\subsection{Hessian of a Neural Network}
Consider a function $\mathcal{L}:\mathbb{R}^{|x_l|} \rightarrow \mathbb{R}$ from the output of the \textsc{nn} to scalar value. This later will be interpreted as loss or likelihood, but for now, let us stick to the general case.

We are interested in the hessian of this scalar value w.r.t. the parameters of the \textsc{nn}
\[
\nabla^2_\phi \Big( \mathcal{L}(f_\phi(x_0)) \Big)
\in\mathfrak{M}
    \left(
    \sum_{p=1}^w|\phi_i|,\sum_{p=1}^w|\phi_i|
    \right).
\]
Similarly to the previous section, it is convenient to see this matrix as block matrices, separated layer-wise
\[
\nabla^2_\phi \mathcal{L}(f_\phi(x_0))
=
\left(\begin{array}{cccccc}
    \nabla^2_{\phi_1} \mathcal{L}(f_\phi(x_0)) 
        & \frac{\partial^2}{\partial\phi_1\partial\phi_2} 
            \mathcal{L}(f_\phi(x_0)) 
        &
        & \dots 
        &
        & \frac{\partial^2}{\partial\phi_1\partial\phi_w} 
            \mathcal{L}(f_\phi(x_0)) \\ 
    \frac{\partial^2}{\partial\phi_2\partial\phi_1}
            \mathcal{L}(f_\phi(x_0)) 
        & \nabla^2_{\phi_2} \mathcal{L}(f_\phi(x_0)) 
        & 
        &
        &
        & \\ 
    & & & & & \\ 
    \vdots 
        & 
        &
        & \,\ddots\, 
        &
        & \\ 
    & & & & & \\ 
    \frac{\partial^2}{\partial\phi_w\partial\phi_1}      
            \mathcal{L}(f_\phi(x_0))
        & 
        &
        &
        & 
        & \nabla^2_{\phi_w} \mathcal{L}(f_\phi(x_0)
\end{array}\right)
\]

The first common assumption is to consider layers to be independent of each other, i.e. 
\[
\frac{\partial^2}{\partial\phi_i\partial\phi_j} \mathcal{L}(f_\phi(x_0)) = 0 \qquad \forall i\not= j
\]

Let us now fix a layer $p=1,\dots,w$ and focus on a single diagonal block.
\[
\nabla^2_{\phi_p} \mathcal{L}(f_\phi(x_0))
\in\mathfrak{M} (|\phi_p|,|\phi_p|).
\]
According to the chain rule
\begin{equation}\label{eq:hessian_chain_rule}
\nabla^2_{\phi_p} \mathcal{L}(f_\phi(x_0))
=
\underbrace{
J_{\phi_p}f_\phi(x_0)^T \cdot 
    \nabla^2_{x_l}\mathcal{L}(x_l) \cdot
    J_{\phi_p}f_\phi(x_0)
}_{=:G(\phi)}
+
\sum_{i=1}^{|x_l|} 
    [\nabla_{x_l}\mathcal{L}(x_l)]_i \cdot
    \nabla^2_{\phi_p} [f_\phi(x_0)]_i
\end{equation}

The second term of the RHS is equal to 0 if the model perfectly fits the dataset, $\nabla_{x_l}\mathcal{L}(x_l)=0$, OR if $f$ is linear in the parameters, $H_{\phi_p} [f_\phi(x_0)]_i=0$.

The first term of the RHS of Eq.~\ref{eq:hessian_chain_rule}, $G(\phi)$, is in literature referred to as Generalized Gauss-Newton (\textsc{ggn}) matrix. It can be computed efficiently thanks to the view of the Jacobian as layer by layer. 
Using equation (\ref{jacobian_chain_rule}), the expression for the approximated hessian w.r.t. to $\phi_p$ is then
\begin{align*}
    G(\phi)
        &=
            J_{\phi_p}f_\phi(x_0)^T \cdot 
            H_{x_l}\mathcal{L}(x_l) \cdot
            J_{\phi_p}f_\phi(x_0) = \\
        &= 
            J_{\phi_p}
                f_{L_{\mathcal{W}(p)}}^T
            \left(\prod_{k=\mathcal{W}(p)+1}^{l} 
                J_{x_k}f_{L_k}^T \right)
            H_{x_l}\mathcal{L}(x_l)
            \left(\prod_{\mathcal{W}(p)+1}^{k=l} 
                J_{x_k}f_{L_k} \right)
            J_{\phi_p}
                f_{L_{\mathcal{W}(p)}}
\end{align*}
and from this, we can build an efficient backpropagation algorithm.

\begin{algorithm}
\caption{Algorithm for $J_{\phi}f^T \cdot \nabla^2\mathcal{L} \cdot J_{\phi}f$}\label{alg:cap}
\begin{algorithmic}
\State $M$ = $\nabla^2_{x_l}\mathcal{L}(x_l)$
\For{$k=l,l-1,\dots,1$}
\If{$L_k$ is parametric with $\phi_p$ (i.e. $k=\mathcal{W}(p)$):}
    \State $H_p$ = $J_{\phi_p} f_{L_k}^\top \cdot M \cdot J_{\phi_p} f_{L_{k}}$ 
\EndIf
\State $M$ = $J_{x_k}f_{L_k}^\top \cdot M \cdot J_{x_k}f_{L_k}$
\EndFor
\State \textbf{return} $(H_1,\dots,H_w)$
\end{algorithmic}
\end{algorithm}

As it is written, each $H_p$ is a matrix $|\phi_p|\times|\phi_p|$ so we technically obtain the hessian in a block-diagonal form 
\[\left(\begin{array}{ccc}
    H_1 & & 0 \\
     & \ddots & \\
    0 & & H_w
\end{array}\right)\]
if we are interested in the diagonal only, we can construct that by concatenation of the diagonals for each $H_p$.

\subsubsection{Fast approximated version of the Algorithm}
The idea is to backpropagate only the diagonal of the matrix $M$, neglecting all the non-diagonal elements. In a single backpropagation step, we have
\[
M' = J_{x_p}f_{L_p}^\top \cdot M \cdot J_{x_p}f_{L_p}
\]
In order to backpropagate the diagonal only, we need to use the operator 
\[
\text{diag}(M) \mapsto \text{diag}(M')
\]
For linear layers and activation functions, this operator is trivial. For the convolutional layer it turns out that this operator is itself a convolution
\[
\text{diag}(M')
= \texttt{conv}_{\text{feat}=\phi^{(2)}_p} (\text{diag}(M))
\]
where the kernel tensor $\phi^{(2)}_p$ is the pointwise square of the kernel tensor $\phi_p$.

\subsection{Hessian of a Reconstruction Loss}
In the previous Algorithm, the backpropagated quantity is initialized as the hessian of the loss w.r.t. the output of the \textsc{nn} $\nabla^2_{x_l}\mathcal{L}(x_l)$, or, in an equivalent but more compact notation $\nabla^2_{f}\mathcal{L}(f)$. The value of this hessian clearly depends on the specific choice of loss function $\mathcal{L}$.

The most common choice of the likelihood for regression is the Gaussian distribution, while for classification it is the Bernoulli distribution. The Gaussian log-likelihood is 
\begin{equation}
\mathcal{L}(f) := \log p (x | \mu\!=\!f_\theta(x), \sigma^2\!=\!\sigma_d^2)
= 
-\frac{1}{2\sigma_d^2} \|x - f_\theta(x) \|^2 -\log(\sqrt{2\pi}\sigma_d)
\end{equation}
and its hessian is identity scaled with $\sigma_d$
\begin{equation}
\nabla^2_f \log p (x | \mu\!=\!f_\theta(x), \sigma^2\!=\!\sigma_d^2)
=
-(\sigma_d)^{-2} \mathbb{I}
\end{equation}

The Bernoulli log-likelihood is
\begin{equation}
\mathcal{L}(f) := \log p (c| f_\theta(x)) 
= 
\log [\texttt{softmax}(f_\theta(x))]_c
=
[f_\theta(x)]_c - \log\left(\sum_i e^{[f_\theta(x)]_i}\right)
\end{equation}
and its hessian can be written in terms of the vector $\pi=\texttt{softmax}(f_\theta(x))$ of predicted probabilities 
\begin{equation}
\nabla^2_f \log p (c| f_\theta(x)) 
=
- \nabla^2_f \log\left(\sum_i e^{[f_\theta(x)]_i}\right)
=
diag(\pi) - \pi\pi^T
\end{equation}

We highlight that both Hessians are independent on the label, and thus the \textsc{ggn} is equal to the Fisher matrix, this is true every time $p(y|f_\theta(x))$ is an exponential family distribution with natural parameters $f_\theta(x)$. In this paper, we focus on the Gaussian log-likelihood, but we emphasize that the method is not limited to this distribution.

\update{\section{Intuition on optimization of variance in VAEs}\label{sec:vae_optim}}

\update{One can imagine the Gaussian VAE as an infinite mixture of Gaussians (see e.g.\@ \citet{mettei2018vaemixture} for an extensive discussion of this link) where the weights are fixed by the prior on the latent space. To increase the probability of the training data $p(x)$, optimizing the neural network will push the probability mass from regions far away from training data to regions with training data. Thus, the network should learn to have large variance far away from training data, and low variance close to training data.}

\begin{figure}
     \centering
     \begin{subfigure}[b]{0.24\textwidth}
         \centering
         \includegraphics[width=\textwidth]{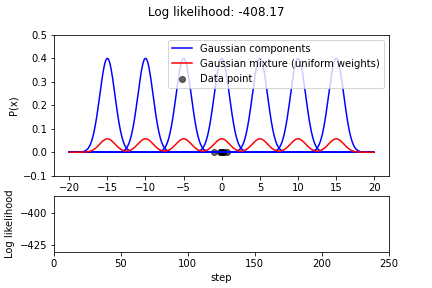}
     \end{subfigure}
     \begin{subfigure}[b]{0.24\textwidth}
         \centering
         \includegraphics[width=\textwidth]{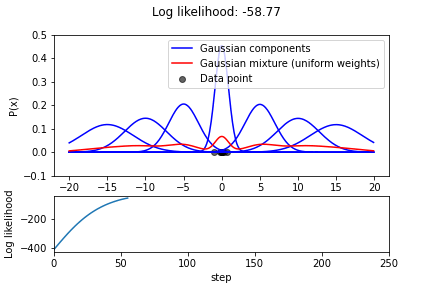}
     \end{subfigure}
     \begin{subfigure}[b]{0.24\textwidth}
         \centering
         \includegraphics[width=\textwidth]{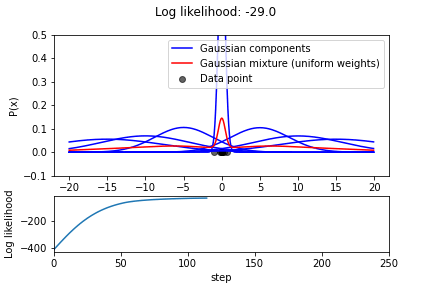}
     \end{subfigure}
     \begin{subfigure}[b]{0.24\textwidth}
         \centering
         \includegraphics[width=\textwidth]{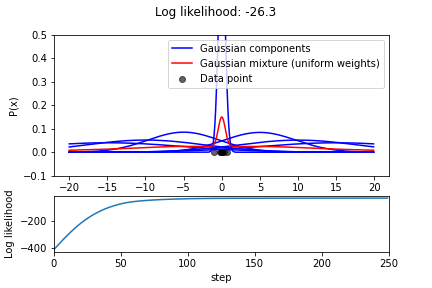}
     \end{subfigure}
        \caption{Snapshots of optimization of a mixture of Gaussians with fixed means and component weights (i.e.\@ only variances are learned). We observe that variances of components far away from data increase in order to push more probability mass to the region where data resides. The Gaussian VAE should exhibit the same behavior in order to maximize data likelihood, but in practice, it does not.}
        \label{fig:vae_optim}
\end{figure}

\update{We illustrate this idea with a toy example (see snapshots in \cref{fig:vae_optim} or animation at \url{https://frederikwarburg.github.io/gaussian_vae.html}). In this example, we show a mixture of Gaussian with components, where we optimize the variance (and fix the mean and the weights of the mixture components). We see that the variance of the components far away from the training data increases, whereas the variance of the component close to the data decreases. The opening example of the paper demonstrates that the Gaussian VAE does not exhibit this behavior even if this is optimal in terms of data likelihood.}

\newpage
\small
\bibliography{neurips_2022}

\newpage

\end{document}